\pdfoutput=1
\documentclass[11pt,twoside]{article}

\newcommand*\TitleRunning{Optimistic Online Convex Optimization in Dynamic Environments}
\newcommand*\AuthorRunning{Qing-xin Meng and Jian-wei Liu}
\newcommand*\Keywords{Optimistic Online Learning, Dynamic Regret, Adaptive Trick, Environmental Adaptive, Meta-algorithm}
\usepackage{arxiv}\usetikzlibrary{arrows.meta,datavisualization.formats.functions}
\title{Optimistic Online Convex Optimization in Dynamic Environments}

\author{\sc Qing-xin Meng
\orcid{0000-0003-4014-7405}
\\\small\email{qingxin6174@gmail.com}
\and\sc Jian-wei Liu
\\\small\email{liujw@cup.edu.cn}
}

\addr{Department of Automation, College of Information Science and Engineering \\
China University of Petroleum-Beijing (CUP), Beijing 102249, China}

\DeclareMathOperator*{\dom}{\operatorname{dom}}
\DeclareMathOperator*{\id}{\operatorname{id}}
\DeclareMathOperator*{\regret}{\operatorname{regret}}

\begin{document}
\maketitle\thispagestyle{firstpage}

\begin{abstract}
In this paper, we study the optimistic online convex optimization problem in dynamic environments.~Existing works have shown that Ader enjoys an $O\big(\sqrt{\left(1+P_T\right)T}\,\big)$ dynamic regret upper bound, where $T$ is the number of rounds, and $P_T$ is the path length of the reference strategy sequence. 
However, Ader is not environment-adaptive. 
Based on the fact that optimism provides a framework for implementing environment-adaptive, we replace Greedy Projection~(GP) and Normalized Exponentiated Subgradient~(NES) in Ader with Optimistic-GP and Optimistic-NES respectively, and name the corresponding algorithm ONES-OGP. 
We also extend the doubling trick to the adaptive trick, and introduce three characteristic terms naturally arise from optimism, namely $M_T$, $\widetilde{M}_T$ and $V_T+1_{L^2\rho\left(\rho+2 P_T\right)\leqslant\varrho^2 V_T}D_T$, to replace the dependence of the dynamic regret upper bound on $T$. 
We elaborate ONES-OGP with adaptive trick and its subgradient variation version, all of which are environment-adaptive. 
\end{abstract}


\section{Introduction}

Consider the following  formalized Online Convex Optimization (OCO) problem \citep{shwartz2012online}. At round $t$, the player chooses the strategy $x_t\in C$ according to some algorithm, where $C$ is closed and convex, and $\rho=\sup\nolimits_{x,y\in C}\left\lVert x-y\right\rVert<+\infty$, the adversary (environment) feeds back a convex loss function $\varphi_t$ with $\dom\partial\varphi_t\supset C$ and $\left\lVert \partial\varphi_t\left(C\right)\right\rVert\leqslant\varrho<+\infty$, 
where $\partial$ represents the subdifferential operator. 
The bilinear map is denoted by $\left\langle\cdot\,,\cdot\right\rangle$. 
Let $H$ be a Hilbert space over $\mathbb{R}$, and assume $C\subset H$. The bilinear map $\left\langle\cdot\,,\cdot\right\rangle$ defined on $H$ represents its inner product. 
We choose the dynamic regret as the performance metric \citep{zinkevich2003online}, that is, 
\begin{equation*}
\regret\left(z_1, z_2,\cdots,z_T\right)\coloneqq\sum_{t=1}^T \varphi_t\left(x_t\right)-\sum_{t=1}^T \varphi_t\left(z_t\right), 
\end{equation*}
where $z_t\in C$ represents the reference strategy in round $t$, and $T$ is the number of rounds. 
For static regret, it suffices to set $z_t\equiv z$. 
There are plenty of works devoted to designing online algorithms to minimize the worst-case static regret $\sup_{z\in C}\regret\left(z, z,\cdots,z\right)$  \citep{bianchi2006prediction, shwartz2012online, hazan2019introduction, orabona2019modern}. 
Recently, designing online algorithms to minimize dynamic regret has attracted much attention \citep{hall2013dynamical, jadbabaie2015online, mokhtari2016online, zhang2018adaptive, zhao2020dynamic, campolongo2021closer, kalhan2021dynamic}. 

The regret upper bound usually contains some characteristic terms, for example, the path length term  \citep{zinkevich2003online}, 
\[P_T=\sum_{t=2}^{T}\left\lVert z_t-z_{t-1}\right\rVert,\]
and the gradient variation term  \citep{chiang2012online}, 
\begin{equation}
\label{gradient-variation}
\begin{aligned}
V_T=\sum_{t=2}^{T}\sup_C\left\lVert \nabla\varphi_t-\nabla\varphi_{t-1}\right\rVert^2.
\end{aligned}
\end{equation}
Usually the dynamic regret upper bound contains the path length term.  
\citet{zinkevich2003online} shows that Greedy Projection~(GP) achieves an $O\big(\left(1+P_T\right)\sqrt{T}\big)$ dynamic regret upper bound. 
\citet{zhang2018adaptive} propose a method, namely adaptive learning for dynamic environment~(Ader), achieves an  $O\big(\sqrt{\left(1+P_T\right)T}\,\big)$ dynamic regret upper bound, which is optimal in completely adversarial environment. 
The main idea of Ader is to run multiple GP in parallel, each with a different step size that is optimal for a specific path length, and track the best one with Normalized Exponentiated Subgradient (NES). 
Actually, Ader is an application of meta-learning techniques, which have become standard since the MetaGrad algorithm was proposed by \citet{erven2016metagrad}. 
\citet{zhao2020dynamic} follow the idea of Ader, and try to utilize smoothness to improve its dynamic regret. 
However, after studying their paper, we argue that the regret upper bound $O\big(\sqrt{\left(1+P_T\right)\left(1+P_T+V_T\right)}\,\big)$ they claim is incorrect. 
Indeed, the regret upper bound obtained by their method cannot escape the dependence on $T$. 
They mistakenly treat $\ln\log_2 T$ as a constant (at the top of page~27, url at \href{https://arxiv.org/abs/2007.03479}{https://arxiv.org/abs/2007.03479}). 
One might argue that treating $\ln\log_2 T$ as a constant is quite reasonable rather than a mistake, e.g. $\ln\log_2 10^{100}<6$. 
If understood in this way, $\ln T$ can be further regarded as a constant, because $\ln 10^{100}<231$, and thus $O\left(\log T\right)$ can be treated as $O\left(1\right)$. 
Obviously this is absurd, because this understanding violates the definition of the asymptotic upper bound notation $O$. 

After in-depth study, we assert that the gradient variation term $V_T$ and its smoothness constraint utilized by \citet{zhao2020dynamic} constitute a special case of optimism. 
By setting the prediction term for impending loss to be the real loss of the previous round, the regret upper bound for Optimistic Greedy Projection (OGP) naturally includes the gradient variation item $V_T$. 
Therefore, in this paper, rather than catering to the gradient variation term $V_T$, we focus on finding suitable characteristic terms naturally induced by optimism to replace the dependence of the dynamic regret upper bound on $T$, which are $O\left(T\right)$ in the worst case while be much smaller in benign environments. 
Note that an online algorithm usually requires a combination of the doubling trick to unfreeze $T$, which means that the doubling trick also needs to be extended. 

The novelties of this article are as follows. 
\begin{itemize}
\item 
We follow the idea of Ader, and replace GP and NES in Ader with OGP and Optimistic Normalized Exponentiated Subgradient (ONES) respectively. 
\item 
In order to replace the dependence of the dynamic regret on $T$, we extend the doubling trick to the adaptive trick, and introduce three characteristic terms naturally arise from optimism, namely $M_T$, $\widetilde{M}_T$ and $V_T+1_{L^2\rho\left(\rho+2 P_T\right)\leqslant\varrho^2 V_T}D_T$ (where $V_T$ is the subgradient variation term, the general form of gradient variation term). 
\end{itemize}

Specifically, we present an algorithm, named ONES-OGP with adaptive trick, which enjoys an $O\big(\sqrt{\left(1+P_T\right)M_T}\,\big)$ dynamic regret upper bound. 
We further improve $M_T$ to $\widetilde{M}_T$ with the help of auxiliary strategies. 
We also present an algorithm, named subgradient variation version of ONES-OGP with adaptive trick, which achieves an $O\Big(\sqrt {\left(1+P_T\right)\big(V_T+1_{L^2\rho\left(\rho+2 P_T\right)\leqslant\varrho^2 V_T}D_T\big)}\,\Big)$ dynamic regret upper bound, and fixes bugs of \citet{zhao2020dynamic}. 
The adaptive trick acts as an outer loop, dividing $M_T$ (or $\widetilde{M}_T$, or $V_T$ and $D_T$) into different stages to run a specific algorithm, just like the doubling trick proposed by \citet{schapire1995gambling}, dividing $T$ into different stages to run a specific algorithm. 

Comparing to $O\big(\sqrt{\left(1+P_T\right)T}\,\big)$, our regret upper bounds replace the dependency on $T$ with $M_T$, $\widetilde{M}_T$ and $V_T+1_{L^2\rho\left(\rho+2 P_T\right)\leqslant\varrho^2 V_T}D_T$ respectively. 
Since these characteristic terms are at most $O\left(T\right)$, our bounds become much tighter when the prediction terms are well-estimated, and safeguard the same guarantee when the environment is adversarial. 
Therefore, all our algorithms are environment-adaptive. 

Note that $M_T$, $\widetilde{M}_T$ and $V_T+1_{L^2\rho\left(\rho+2 P_T\right)\leqslant\varrho^2 V_T}D_T$ are all characteristic terms induced by optimism, we argue that optimism is the driving force behind. 

\section{Optimistic Algorithms}
\label{OA}

In this section, we present regret upper bounds for OGP and ONES. 
Even though \cref{OGP-regret} and \cref{ONES-regret} can be proved by a unified framework (See Section~6 of \citealp{meng2021unified}), for the sake of completeness, we provide direct proofs of these two lemmas in \cref{proof-OGP-regret} and \cref{proof-ONES-regret} respectively. 
Before the formal elaboration, we briefly review optimism and its properties. 

An algorithm is said to be optimistic if its update rule contains a prediction term for the impending loss. 
The optimistic mirror descent was proposed by \citet{chiang2012online} and extended by \citet{rakhlin2013online}. 
It is usually formalized as 
\begin{equation*}
\begin{aligned}
\widetilde{x}_{t+1}&=\arg\min_{x\in C}\left\langle x_{t}^* ,\, x\right\rangle+\frac{1}{\eta}B_{\psi}\left(x, \widetilde{x}_{t}\right), \quad x_{t}^*\in\partial\varphi_t\left(x_{t}\right),\\
x_{t+1}&=\arg\min_{x\in C}\left\langle \widehat{x}_{t+1}^* ,\, x\right\rangle+\frac{1}{\eta}B_{\psi}\left(x, \widetilde{x}_{t+1}\right),
\end{aligned}
\end{equation*}
where $B_{\psi}$ represents the Bregman divergence w.r.t. $\psi$, and $\widehat{x}_{t}^*$ denotes the estimated linear function for the impending loss. 
The projection form of optimistic mirror descent is also called the Optimistic Greedy Projection (OGP). 
Optimistic mirror descent with $\psi$ being negative entropy is also known as Optimistic Normalized Exponentiated Subgradient~(ONES, or Optimistic-Hedge). 

Optimism also provides a framework for implementing environment-adaptive, i.e., maintaining some regret upper bound when the environment is adversarial, and being able to tighten the upper bound when the environment is not fully adversarial. 
It is worth mentioning that how to predict impending losses is not the focus of optimistic algorithms, even if the performance to environment-adaptive depends on the accuracy of prediction terms.
For a non-optimistic algorithm, it suffices to set prediction terms to be null. 

\subsection{Optimistic Greedy Projection}
\label{OGP}

Optimistic Greedy Projection (OGP) can be formalized as the following workflow, 
\begin{equation}
\label{OGP-iteration}
\begin{aligned}
\widetilde{x}_{t+1}&=P_C\left(\widetilde{x}_t-\eta x_t^*\right), && x_t^*\in\partial\varphi_t\left(x_t\right),\quad\widetilde{x}_1\in C,\\
x_{t+1}&=P_C\left(\widetilde{x}_{t+1}-\eta\widehat{x}_{t+1}^*\right),&& 
\end{aligned}
\end{equation}
where $P_C$ represents the projection onto the subset $C$, $\eta>0$ is the step size, $\widehat{x}_{t}^*\in H$ is the estimated linear loss function in round $t$. 
Note that OGP allows arbitrary $\widehat{x}_{t}^*$.  
In Hilbert space, the projection of any point onto a closed convex subset exists and is unique (See \cref{best-approx}), which leads to $\widetilde{x}_{t}, x_{t}\in C$, $\forall t\in\mathbb{N}$. 

\citet{chiang2012online} and \citet{zhao2020dynamic} studied the static and dynamic regret for OGP respectively under the assumptions that $\varphi_t$ is differentiable and $\widehat{x}_{t+1}^*=\nabla\varphi_t\left(\widetilde{x}_{t+1}\right)$. 
The following lemma states that OGP without any restriction has dynamic regret upper bound. 

\begin{lemma}
\label{OGP-regret}
OGP enjoys the following dynamic regret upper bound,
\begin{equation}
\label{OGP-regret-bound}
\begin{aligned}
\regret\left(z_1, z_2,\cdots,z_T\right)
\leqslant\frac{\rho^2}{2\eta}+\frac{\rho}{\eta}\sum_{t=2}^{T}\left\lVert z_t-z_{t-1}\right\rVert 
+\frac{\eta}{2}\sum_{t=1}^{T}\left\lVert x_t^*-\widehat{x}_t^*\right\rVert^2-\frac{1}{2\eta}\sum_{t=1}^{T}\left\lVert x_{t}-\widetilde{x}_{t}\right\rVert^2,
\end{aligned}
\end{equation}
where $z_t\in C$ represents the reference strategy in round $t$. 
\end{lemma}

\begin{remark}
The novelty of \cref{OGP-regret} over \citet{zhao2020dynamic} is that the third term of \cref{OGP-regret-bound} corresponds to $\widehat{x}_t^*$ rather than the $t-1$ subgradient. 
We emphasize optimism without assuming smoothness. 
We can further assume that $\widehat{x}_t^*$ is the subgradient of an estimated convex loss $\widehat{\varphi}_t$, and $\partial\widehat{\varphi}_t$ is Lipschitz continuous, then the dynamic regret for OGP has subgradient variation type (See \cref{OGP-corollary}).
\end{remark}

By introducing an appropriate auxiliary strategy sequence, the dynamic regret for OGP can be formalized as the following corollary. 

\begin{corollary}
\label{auxiliary-OGP-regret}
OGP enjoys the following dynamic regret upper bound,
\begin{equation}
\label{auxiliary-OGP-regret-bound}
\begin{aligned}
\regret\left(z_1, z_2,\cdots,z_T\right)
\leqslant\frac{\rho^2}{2\eta}+\frac{\rho}{\eta}\sum_{t=2}^{T}\left\lVert z_t-z_{t-1}\right\rVert 
+\frac{\eta}{2}\sum_{t=1}^{T}h_{\left\lVert x_t^*\right\rVert}\left(\left\lVert x_t^*-\widehat{x}_t^*\right\rVert\right), 
\end{aligned}
\end{equation}
where 
$h_\sigma\left(\delta\right)=\delta^2-\left(\left\lvert\delta\right\rvert-\left\lvert\sigma\right\rvert\right)_+^2$, and $\frac{1}{2}h$
represents the robust Huber penalty \citep{huber1964robust}. 
\end{corollary}

\begin{remark}
\cref{auxiliary-OGP-regret} is the dynamic version of Theorem~3 of \citet{flaspohler2021online}. 
\end{remark}

Set $\widehat{x}_t^*$ to be null, then OGP degenerates into GP, that is, 
\begin{equation*}
\begin{aligned}
x_{t+1}=P_C\left(x_t-\eta x_t^*\right), \quad x_t^*\in\partial\varphi_t\left(x_t\right),\quad x_1\in C, 
\end{aligned}
\end{equation*}
and the corresponding dynamic regret upper bound degenerates into the following form,
\begin{equation}
\label{MD-regret-bound}
\frac{\rho^2}{2\eta}+\frac{\rho}{\eta}\sum_{t=2}^{T}\left\lVert z_t-z_{t-1}\right\rVert+\frac{\eta}{2}\sum_{t=1}^{T}\left\lVert x_t^*\right\rVert^2, 
\end{equation}
which is a slight improvement of the following well-known upper bound \cite{zinkevich2003online, zhang2018adaptive}, 
\begin{equation*}
\frac{7\rho^2}{4\eta}+\frac{\rho}{\eta}\sum_{t=2}^T\lVert z_{t}-z_{t-1}\rVert+\frac{\eta}{2}\sum_{t=1}^T\left\lVert x_t^*\right\rVert^2.
\end{equation*}

Comparing \cref{OGP-regret-bound,auxiliary-OGP-regret-bound,MD-regret-bound}, we argue that by introducing the estimated linear loss function  $\widehat{x}_{t}^*$, the dynamic regret upper bounds  can be tighter in the case the environment is not fully adversarial and $\widehat{x}_{t}^*$ is well-estimated, and meanwhile guarantees the same upper bound in the worst case.

The following corollary states that, 
under the assumptions that $\widehat{x}_t^*$ is the subgradient of the estimated convex loss $\widehat{\varphi}_t$, and $\partial\widehat{\varphi}_t$ is Lipschitz continuous, 
the dynamic regret upper bound for OGP has subgradient variation type.

\begin{corollary}
\label{OGP-corollary}
If $\widehat{x}_t^*\in\partial\widehat{\varphi}_t\left(\widetilde{x}_{t}\right)$ and $\partial\widehat{\varphi}_t$ is Lipschitz continuous, i.e., $\exists L>0$, such that 
\begin{equation*}
\left\lVert\partial\widehat{\varphi}_t\left(x\right)-\partial\widehat{\varphi}_t\left(y\right)\right\rVert\leqslant L\left\lVert x-y\right\rVert,\quad\forall x,y\in C,
\end{equation*} 
where $\widehat{\varphi}_t$ represents the estimated convex loss function, then OGP enjoys the following dynamic regret upper bound, 
\begin{equation*}
\begin{aligned}
\regret\left(z_1, z_2,\cdots,z_T\right)
\leqslant\frac{\rho\left(\rho+2 P_T\right)}{2\eta}
+\eta\sum_{t=1}^{T}\left(\sup_{x\in C}\left\lVert x^{\varphi_t}-x^{\widehat{\varphi}_{t}}\right\rVert^2+1_{\eta>\frac{1}{\sqrt{2}L}}L^2 \left\lVert x_{t}-\widetilde{x}_{t}\right\rVert^2\right),
\end{aligned}
\end{equation*}
where $x^{\varphi_t}\in\partial\varphi_t\left(x\right)$, $x^{\widehat{\varphi}_{t}}\in\partial\widehat{\varphi}_{t}\left(x\right)$, and $1_{\eta>\frac{1}{\sqrt{2}L}}$ is the zero-one indicator function w.r.t. $1_{\eta>\frac{1}{\sqrt{2}L}}=1$ iff $\eta>\frac{1}{\sqrt{2}L}$. 
\end{corollary}

\begin{remark}
\label{remark:Lipschitz}
The novelty of \cref{OGP-corollary} over \citet{zhao2020dynamic} is that $\partial\widehat{\varphi}_t$ is Lipschitz continuous rather than $\partial\varphi_t$. 
In particular, choose $\widehat{\varphi}_t=\varphi_{t-1}$ and set $\partial\varphi_t$ to be Lipschitz continuous, then \cref{OGP-corollary} degenerates into the case of \citet{zhao2020dynamic}. 
Note that $\varphi_t$ is the adversary's feedback and $\widehat{\varphi}_t$ is the estimated loss, in order to maintain the OCO settings, we impose restrictions on the estimated loss, and try not to limit the adversary's feedback. 
One may suspect that if real losses are not Lipschitz continuous, but the predictions are (for fixed $L$), the closeness of these predictions to the losses will be very poor. 
In fact, this intuition doesn't always hold true. See \cref{Counterexample} for a counterexample of this intuition. 
\end{remark}

\subsection{Optimistic Normalized Exponentiated Subgradient}
\label{ONES}

Optimistic Normalized Exponentiated Subgradient~(ONES)  can be formalized as the following workflow, 
\begin{equation}
\label{ONES-iteration}
\begin{aligned}
\widetilde{w}_{t+1}&=\mathscr{N}\big(\widetilde{w}_{t}\circ\mathrm{e}^{-\theta\ell_{t}}\big),&&\widetilde{w}_{1}\in\bigtriangleup^{n}, \\
w_{t+1}&=\mathscr{N}\big(\widetilde{w}_{t+1}\circ\mathrm{e}^{- \theta\widehat{\ell}_{t+1}}\big),&&
\end{aligned}
\end{equation}
where $\mathscr{N}$ is the normalization operator, $\circ$ represents the Hadamard product, $\theta>0$ is the step size, $\ell_t$ is the loss vector, $\widehat{\ell}_t$ is the corresponding estimated vector, and $\bigtriangleup^{n}\coloneqq\left\{w\left\lvert\,w\in \mathbb{R}_+^{n+1},\,\left\lVert w\right\rVert_1=1\right.\right\}$ is the probability simplex. 
Similar to OGP, ONES allows arbitrary $\widehat{\ell}_t$. 
The normalization operator $\mathscr{N}$ guarantees that $\widetilde{w}_{t}, w_{t}\in \bigtriangleup^{n}$. 

\begin{remark}
ONES (Equation~\ref{ONES-iteration}) is equivalent to 
\begin{equation*}
\label{ONES-equivalent}
\begin{aligned}
\widetilde{v}_{t+1}&=\widetilde{v}_{t}-\theta\ell_{t}, &&
\widetilde{w}_{t+1}=\mathscr{N}\mathrm{e}^{\widetilde{v}_{t+1}}, \\
v_{t+1}&=\widetilde{v}_{t+1}-\theta\widehat{\ell}_{t+1}, &&
w_{t+1}=\mathscr{N}\mathrm{e}^{v_{t+1}},
\end{aligned}
\end{equation*}
or the following compact version, 
\begin{equation*}
w_{t+1}=\mathscr{N}\left(\widetilde{w}_{1}\circ\mathrm{e}^{-\theta\sum_{i=1}^{t}\ell_i- \theta\widehat{\ell}_{t+1}}\right).
\end{equation*}
\end{remark}

The following lemma states that ONES has static regret upper bound. 

\begin{lemma}
\label{ONES-regret}
ONES enjoys the following static regret upper bound,
\begin{equation}
\label{ONES-regret-bound}
\begin{aligned}
\regret\left(w, w, \cdots, w\right)\leqslant\frac{1}{\theta}\sum_{i}w\left(i\right)\ln \frac{w\left(i\right)}{\widetilde{w}_1\left(i\right)}
+\frac{\theta}{2}\sum_{t=1}^{T}\left\lVert \ell_t-\widehat{\ell}_t\right\rVert_{\infty}^2-\frac{1}{2\theta}\sum_{t=1}^T\left\lVert w_t-\widetilde{w}_{t}\right\rVert_1^2,
\end{aligned}
\end{equation}
where $w\in\bigtriangleup^{n}$ represents the reference strategy.
\end{lemma}

\begin{remark}
\cref{ONES-regret} is a refined version of Theorem 19 of \citet{syrgkanis2015fast}. 
Without the subtraction term, \cref{ONES-regret} is directly implied by Theorem~7.28 of \citet{orabona2019modern}. 
\end{remark}

By introducing an appropriate auxiliary strategy sequence, the static regret for ONES can be formalized as the following corollary. 

\begin{corollary}
\label{auxiliary-ONES-regret}
ONES enjoys the following static regret upper bound,
\begin{equation}
\label{auxiliary-ONES-regret-bound}
\begin{aligned}
\regret\left(w, w, \cdots, w\right)
\leqslant\frac{1}{\theta}\sum_{i}w\left(i\right)\ln \frac{w\left(i\right)}{\widetilde{w}_1\left(i\right)}
+\frac{\theta}{2}\sum_{t=1}^{T}h_{\left\lVert \ell_t\right\rVert_{\infty}}\left(\left\lVert \ell_t-\widehat{\ell}_t\right\rVert_{\infty}\right).
\end{aligned}
\end{equation}
\end{corollary}

\begin{remark}
The proof of \cref{auxiliary-ONES-regret} is very similar to the proof of \cref{auxiliary-OGP-regret} (See \cref{pf:auxiliary-OGP-regret}). 
\cref{auxiliary-ONES-regret} can also be directly implied by Theorem~3 of \citet{flaspohler2021online}. 
\end{remark}

Set $\widehat{\ell}_t$ to be null, then ONES degenerates into NES, that is, 
\begin{equation*}
\begin{aligned}
 w_{t+1}=\mathscr{N}\big(w_{t}\circ\mathrm{e}^{-\theta\ell_{t}}\big),\quad w_1\in \bigtriangleup^{n-1},
\end{aligned}
\end{equation*}
and the corresponding static regret upper bound degenerates into the following form,
\begin{equation}
\label{NES-regret-bound}
\frac{1}{\theta}\sum_{i}w\left(i\right)\ln \frac{w\left(i\right)}{w_1\left(i\right)}+\frac{\theta}{2}\sum_{t=1}^{T}\left\lVert \ell_t\right\rVert_{\infty}^2. 
\end{equation}
NES has a plethora of different names (Hedge, exponentially weighted average, etc). 
A well-known upper bound for NES is to set $w_1=\frac{1}{n+1}\mathbf{1}$ in \cref{NES-regret-bound}, where $\mathbf{1}$ is the all-ones vector in $\mathbb{R}^{n+1}$ \citep{shwartz2012online}. 

Comparing \cref{ONES-regret-bound,auxiliary-ONES-regret-bound,NES-regret-bound}, we argue that by introducing the estimated linear loss vector $\widehat{\ell}_t$, the static regret upper bound  can be tighter in the case the environment is not fully adversarial and $\widehat{\ell}_t$ is well-estimated, and meanwhile guarantees the same upper bound in the worst case.

A typical application scenario for ONES is as a meta-algorithm to track the best experts by combining their advice.~Suppose that a group of experts $\left\{e_i\right\}_{i\in E}$ provide suggestions to a player, where $E$ is an appropriate index set.  
At round $t$, the expert $e_i$ provides a suggestion strategy $x_t\left(i\right)\in C$, the player combines their suggestions with weight $w_t$ to generate the final strategy $\overline{x}_t=\left\langle w_t, \boldsymbol{x}_t\right\rangle$, where $\boldsymbol{x}_t=\left\{x_t\left(i\right)\right\}_{i\in E}$ and $w_t$ is generated by ONES. Then 
\begin{equation*}
\begin{aligned}
\sum_{t=1}^{T}\varphi_t\left(\overline{x}_t\right)-\varphi_t\left(\left\langle w, \boldsymbol{x}_t\right\rangle\right)
\leqslant\sum_{t=1}^{T}\big\langle\partial\varphi_t\left(\overline{x}_t\right),\,\left\langle w_t-w, \boldsymbol{x}_t\right\rangle\big\rangle 
=\sum_{t=1}^{T}\big\langle\left\langle\partial\varphi_t\left(\overline{x}_t\right), \boldsymbol{x}_t\right\rangle, w_t-w\big\rangle.
\end{aligned}
\end{equation*}
Choose $\ell_t\in\left\langle\partial\varphi_t\left(\overline{x}_t\right), \boldsymbol{x}_t\right\rangle$ as the surrogate linear loss, we have 
\begin{equation}
\label{surrogate}
\begin{aligned}
\sum_{t=1}^{T}\varphi_t\left(\overline{x}_t\right)-\varphi_t\left(\left\langle w, \boldsymbol{x}_t\right\rangle\right)
\leqslant\sum_{t=1}^{T}\left\langle\ell_t, w_t-w\right\rangle.
\end{aligned}
\end{equation}

The following corollary states that, 
under the assumptions that $\partial\widehat{\varphi}_t$ is Lipschitz continuous and $\widehat{\ell}_t\in\left\langle\partial\widehat{\varphi}_t\left(\widetilde{\overline{x}}_t\right), \boldsymbol{x}_t\right\rangle$, where $\widetilde{\overline{x}}_t=\left\langle \widetilde{w}_t, \boldsymbol{x}_t\right\rangle$, 
the static regret upper bound for ONES has subgradient variation type. 

\begin{corollary}
\label{ONES-corollary}
If $\partial\widehat{\varphi}_t$ is $L$-Lipschitz continuous, and 
$\ell_t\in\left\langle\partial\varphi_t\left(\overline{x}_t\right), \boldsymbol{x}_t\right\rangle$, 
$\widehat{\ell}_t\in\left\langle\partial\widehat{\varphi}_t\left(\widetilde{\overline{x}}_t\right), \boldsymbol{x}_t\right\rangle$, 
where $\overline{x}_t=\left\langle w_t, \boldsymbol{x}_t\right\rangle$ and $\widetilde{\overline{x}}_t=\left\langle \widetilde{w}_t, \boldsymbol{x}_t\right\rangle$, 
then ONES enjoys the following static regret upper bound, 
\begin{equation*}
\begin{aligned}
\regret\left(w, w, \cdots, w\right)\leqslant\ &\frac{1}{\theta}\sum_{i}w\left(i\right)\ln \frac{w\left(i\right)}{\widetilde{w}_1\left(i\right)} \\
&+\theta\rho^2\sum_{t=1}^T\left(\sup_{x\in C}\left\lVert x^{\varphi_t}-x^{\widehat{\varphi}_{t}}\right\rVert^2+1_{\theta>\frac{1}{\sqrt{2}\rho^2 L}}\rho^2 L^2\left\lVert w_t-\widetilde{w}_{t}\right\rVert_1^2\right).
\end{aligned}
\end{equation*}
where $x^{\varphi_t}\in\partial\varphi_t\left(x\right)$ and $x^{\widehat{\varphi}_{t}}\in\partial\widehat{\varphi}_{t}\left(x\right)$. 
\end{corollary}

\begin{remark}
We assume that $\partial\widehat{\varphi}_t$ is Lipschitz continuous rather than $\partial\varphi_t$. 
See \cref{remark:Lipschitz} for the reason. 
Note that the subgradient variation term $\sum_{t=1}^T\sup_{x\in C}\left\lVert x^{\varphi_t}-x^{\widehat{\varphi}_{t}}\right\rVert^2$ appears in both \cref{OGP-corollary} and \cref{ONES-corollary}, which satisfies the meta-learning condition. 
\end{remark}

\section{Dynamic Regret}
\label{DR}

In this section, we follow the idea of Ader \cite{zhang2018adaptive}, and enhance the dynamic regret by replacing GP and NES in Ader with OGP and ONES respectively and combining the adaptive trick.~We focus on the simple case of dropping subtraction terms from \cref{OGP-regret-bound} (in \cref{OGP-regret}) and \cref{ONES-regret-bound} (in \cref{ONES-regret}). 
We also introduce an elegant characteristic item $M_T$, which is a measure of estimation accuracy. 

Let's maintain a group of experts $\left\{e_i\right\}_{i\in E}$ ($E$ is unknown temporarily), where the expert $e_i$ operates OGP with a certain parameter $\eta_i$, and then composite their suggestions by weight $w_t$ to obtain the final strategy, i.e., $\overline{x}_t=\left\langle w_t, \boldsymbol{x}_t\right\rangle$, where $\boldsymbol{x}_t=\left\{x_t\left(i\right)\right\}_{i\in E}$, $x_t\left(i\right)$ represents the suggestion of the expert $e_i$, and $w_t$ is generated by ONES. 
Note that we need to replace the dependence on $T$ with some characteristic terms, which means that we need to match the regret upper bounds between OGP and ONES, and then extend the doubling trick. 
We emphasize that the following three key steps are different from Ader. 

\subsection*{Step 1: Match of the Regret Upper Bounds}

Note that the dynamic regret can be decomposed as 
\begin{equation}
\label{regret-decomposition}
\begin{aligned}
\sum_{t=1}^{T}\varphi_t\left(\overline{x}_t\right)-\varphi_t\left(z_t\right)
&=\sum_{t=1}^{T}\varphi_t\left(\left\langle w_t,\boldsymbol{x}_t\right\rangle\right)-\varphi_t\left(\left\langle 1_j,\boldsymbol{x}_t\right\rangle\right)
+\sum_{t=1}^{T}\varphi_t\left(x_t\left(j\right)\right)-\varphi_t\left(z_t\right) \\
&\leqslant\sum_{t=1}^{T}\left\langle \ell_t, w_t-1_j\right\rangle+\sum_{t=1}^{T}\varphi_t\left(x_t\left(j\right)\right)-\varphi_t\left(z_t\right), 
\end{aligned}
\end{equation}
where $1_j$ is the one-hot vector corresponding to the expert $e_j$, $\ell_t\in\left\langle\partial\varphi_t\left(\overline{x}_t\right), \boldsymbol{x}_t\right\rangle$. 
The ``$\leqslant$'' follows from \cref{surrogate}. 
The first term of \cref{regret-decomposition} is the regret for ONES, and the last term of \cref{regret-decomposition} is the regret for expert $e_j$. 

We modify the regret bounds for OGP (Equation~\ref{OGP-regret-bound}) and ONES (Equation~\ref{ONES-regret-bound}) by dropping their subtraction terms, and match the their bounds by introducing $M_T$, that is, 
\begin{align}
\sum_{t=1}^{T}\varphi_t\left(x_t\left(j\right)\right)-\varphi_t\left(z_t\right)
&\leqslant\frac{\rho\left(\rho+2 P_T\right)}{2\eta_j}+\frac{\eta_j\varrho^2}{2}Q_T\left(j\right)
\leqslant\frac{\rho\left(\rho+2 P_T\right)}{2\eta_j}+\frac{\eta_j\varrho^2}{2}M_T,\notag\\
\sum_{t=1}^{T}\left\langle \ell_t, w_t-1_j\right\rangle
&\leqslant\frac{-\ln \widetilde{w}_1\left(j\right)}{\theta}+\frac{\theta\rho^{2}\varrho^{2}}{2}L_T 
\leqslant\frac{-\ln \widetilde{w}_1\left(j\right)}{\theta}+\frac{\theta\rho^{2}\varrho^{2}}{2}M_T,\label{ONES-regret-bound-rewrite}
\end{align}
where 
\begin{equation}
\label{MT-hostility}
\begin{aligned}
P_T&=\sum_{t=2}^{T}\left\lVert z_t-z_{t-1}\right\rVert,
&Q_T\left(j\right)&=4+\varrho^{-2}\sum_{t=1}^{T-1}\left\lVert x_t^*\left(j\right)-\widehat{x}_t^*\left(j\right)\right\rVert^2,\\
L_T&=4+\rho^{-2}\varrho^{-2}\sum_{t=1}^{T-1}\left\lVert \ell_t-\widehat{\ell}_t\right\rVert_{\infty}^2,
&M_T&=\max\left\{L_T, \max_j Q_T\left(j\right)\right\}, 
\end{aligned}
\end{equation}
$x_t^*\left(j\right)\in\partial\varphi_t\left(x_t\left(j\right)\right)$, $x_t\left(j\right)$ is the suggestion strategy of $e_j$, $\widehat{x}_{t}^*\left(j\right)$ is the corresponding estimated linear loss function for  $e_j$ with $\left\lVert \widehat{x}_{t}^*\left(j\right)\right\rVert\leqslant\varrho$, and $\widehat{\ell}_{t}$ is the estimated loss vector with $\left\lVert \widehat{\ell}_{t}\right\rVert\leqslant\rho\varrho$. 

We call $M_T$ as a measure of estimation accuracy since 
\begin{equation*}
\begin{aligned}
M_T=\mathit{\Theta}\left(1+\sum_{t=1}^{T}\left\lVert \ell_t-\widehat{\ell}_t\right\rVert_{\infty}^2+\max_j\sum_{t=1}^{T}\left\lVert x_t^*\left(j\right)-\widehat{x}_t^*\left(j\right)\right\rVert^2\right)
\leqslant O\left(T\right). 
\end{aligned}
\end{equation*}
When the environment is not completely adversarial and all $\widehat{x}_t^*\left(j\right)$ and $\widehat{\ell}_t$ are predicted accurately, then $M_T$ grows slowly. On the contrary, when the environment is completely adversarial, all predictions fail and $M_T$ grows linearly. 

We emphasize that how to predict $\widehat{x}_t^*\left(j\right)$ and $\widehat{\ell}_t$ is not within the analysis scope of our algorithm. 

\subsection*{Step 2: Allocation of the group of experts} 

The main result of this step is summarized as the following theorem. 

\begin{theorem}
\label{expert-group}
Let $M_T$ be fixed, and let 
\begin{equation*}
\begin{aligned}
E=\left\{0, 1, \cdots, \left\lfloor\log_2\sqrt{2T-1}\right\rfloor\right\}.
\end{aligned}
\end{equation*}
The expert $e_i$ operates OGP with $\eta_i=\frac{\rho}{\varrho\sqrt{M_T}}2^i$, $\forall i\in E$. If $\theta\propto\frac{1}{\sqrt{M_T}}$, where $\theta$ is the parameter of ONES, then we have 
\begin{equation*}
\begin{aligned}
\sum_{t=1}^{T}\varphi_t\left(\overline{x}_t\right)-\varphi_t\left(z_t\right)
<O\left(\sqrt{\left(1+P_T\right)M_T}\right).
\end{aligned}
\end{equation*}
\end{theorem}

\begin{remark}
The allocation of the group of experts depends on the range of the optimal parameter $\dot{\eta}$, that is, 
\begin{equation*}
\dot{\eta}=\sqrt{\frac{\rho\left(\rho+2 P_T\right)}{\varrho^2 M_T}}\in\frac{\rho}{\varrho\sqrt{M_T}}\left[1,\,\sqrt{2T-1}\right].
\end{equation*}
The denominator of parameter $\eta_i=\frac{\rho}{\varrho\sqrt{M_T}}2^i$ contains $M_T$~(instead of $T$), resulting in cancellation of denominators on both sides of $\eta_j\leqslant\dot{\eta}$, thereby eliminating the term $\ln\log_2 T$ in the upper bound. 
See \cref{pf:expert-group} for details. 
\end{remark}

\cref{expert-group} states that $T$ is successfully replaced by $M_T$. 
We call the above algorithm ONES-OGP, that is, OGP is the expert algorithm, and ONES is the meta-algorithm.  

\subsection*{Step 3: Extension of the doubling trick}

Note that \cref{expert-group} is based on the premise that $M_T$ is fixed, we utilize the following adaptive trick to unfreeze $M_T$, just like utilizing the doubling trick to unfreeze $T$ to anytime. 

\begin{theorem}[Adaptive Trick]
\label{adaptive-trick}
The adaptive trick 
\begin{itemize}
\item[] calls ONES-OGP with $\theta\propto 2^{-m}$ and $\eta_i=\frac{\rho}{\varrho}2^{i-m}$ for $i=0, 1, \cdots, n$,  \item[] under the constraints that $M_T\in\big[4^m, 4^{m+1}\big)$ and $T\in\frac{1}{2}\big[4^n, 4^{n+1}\big)+1$, 
\end{itemize}
where $m$ indicates the stage index of the game. 
The above execution process achieves an $O\big(\sqrt {\left(1+P_T\right)M_T}\,\big)$  dynamic regret upper bound. 
\end{theorem}

The idea of adaptive trick is to divide the range of $M_T$ into stages of exponentially increasing size and runs ONES-OGP on each stage. 
This is an extension of the doubling trick, which divides $T$ into stages of doubling size and runs some appropriate algorithm on each stage.  
Shifting from monitoring 
$T$ to monitoring 
$M_T$ is a crucial step in achieving environment-adaptive. 

Note that the sublinear dynamic regret for ONES-OGP with adaptive trick holds under $P_T\leqslant o\left(T^2 / M_T\right)$, if $M_T$ grows sublinearly, then $P_T\leqslant O\left(T\right)$, that is, the sublinear dynamic regret holds for arbitrary reference strategy sequence. 

To be understood easy, we illustrate the specific execution process for ONES-OGP with adaptive trick in \cref{alg-1}. 

\begin{figure}[!ht]
\begin{algorithm}[H]
\caption{ONES-OGP with adaptive trick}
\label{alg-1}
\begin{algorithmic}[1]
\STATE $m\gets -1$, \,$n\gets -1$
\FOR{round $t=1, 2, \cdots$}
\STATE $n\gets\left\lfloor\log_2\sqrt{2t-1}\right\rfloor$, $m\gets\left\lfloor\log_4 M_t\right\rfloor$, where $M_t$ is calculated according to \cref{MT-hostility}
\IF{$m$ changed or $n$ changed}
\STATE Construct a set of experts $\left\{e_i\right\}_{i=0}^{n}$ and invoke \cref{alg-3} with $\eta_i=\frac{\rho}{\varrho}2^{i-m}$ for $e_i$
\STATE Call \cref{alg-2} with parameter $n$ and $\theta\propto 2^{-m}$
\ENDIF
\STATE Receive the estimated loss vector $\widehat{\ell}_t$ from an arbitrary estimating process and send it to \cref{alg-2}, 
receive a group of estimated linear losses $\left\{\widehat{x}_{t}^*\left(0\right), \widehat{x}_{t}^*\left(1\right), \cdots, \widehat{x}_{t}^*\left(n\right)\right\}$ from an arbitrary estimating process and send them to each expert
\STATE Get expert advice strategies $\boldsymbol{x}_t=\left\{x_t\left(0\right), x_t\left(1\right), \cdots, x_t\left(n\right)\right\}$, call \cref{alg-2} to get the weight $w_t$
\STATE Output strategy $\overline{x}_t=\left\langle w_t, \boldsymbol{x}_t\right\rangle$, and then observe loss function $\varphi_t$
\STATE Send $\ell_t\in\left\langle\partial\varphi_t\left(\overline{x}_t\right), \boldsymbol{x}_t\right\rangle$ to \cref{alg-2}, send $\partial\varphi_t$ to each expert
\ENDFOR
\end{algorithmic}
\end{algorithm}
\begin{algorithm}[H]
\caption{Subprogram: ONES with parameter $n$ and $\theta$}
\label{alg-2}
\begin{algorithmic}[1]
\REQUIRE $\ell_\tau$ and $\widehat{\ell}_{\tau+1}$ from \cref{alg-1}
\ENSURE $w_{\tau+1}\left(i\right)$, $i=0,1,\cdots, n$
\STATE $\widetilde{w}_1\left(i\right)=\beta\left(i+2\right)^{-\alpha}$, $i=0,1,\cdots, n$, and each call follows the ONES (Equation~\ref{ONES-iteration})
\end{algorithmic}
\end{algorithm}
\begin{algorithm}[H]
\caption{Subprogram: OGP with parameter $\eta$}
\label{alg-3}
\begin{algorithmic}[1]
\REQUIRE $\partial\varphi_\tau$ and $\widehat{x}_{\tau+1}^*$ from \cref{alg-1}
\ENSURE Each call follows the OGP (Equation~\ref{OGP-iteration})
\end{algorithmic}
\end{algorithm}
\end{figure}

\section{Dynamic Regret with Auxiliary Strategies}
\label{DRwAS}

In this section, we illustrate that the characteristic term $M_T$ can be further improved by introducing some appropriate auxiliary strategy sequences. 

\cref{DR} focuses on the simple case of dropping subtraction terms from \cref{OGP-regret-bound} (in \cref{OGP-regret}) and \cref{ONES-regret-bound} (in \cref{ONES-regret}). 
Replace \cref{OGP-regret} and \cref{ONES-regret} with \cref{auxiliary-OGP-regret} and \cref{auxiliary-ONES-regret} respectively, we have that 
\begin{equation*}
\begin{aligned}
\widetilde{M}_T=\max\left\{\widetilde{L}_T, \max_j \widetilde{Q}_T\left(j\right)\right\}, 
\end{aligned}
\end{equation*}
where 
\begin{equation*}
\begin{aligned}
\widetilde{L}_T&=4+\rho^{-2}\varrho^{-2}\sum_{t=1}^{T-1}h_{\left\lVert \ell_t\right\rVert_{\infty}}\left(\left\lVert \ell_t-\widehat{\ell}_t\right\rVert_{\infty}\right), \\
\widetilde{Q}_T\left(j\right)&=4+\varrho^{-2}\sum_{t=1}^{T-1}h_{\left\lVert x_t^*\left(j\right)\right\rVert}\left(\left\lVert x_t^*\left(j\right)-\widehat{x}_t^*\left(j\right)\right\rVert\right). 
\end{aligned}
\end{equation*}
Obviously, $\widetilde{M}_T\leqslant M_T$. 
We also call $\widetilde{M}_T$ as a measure of estimation accuracy. 
For its corresponding algorithm, it suffices to replace $M_T$ in \cref{alg-1} with $\widetilde{M}_T$. 

\section{Dynamic Regret in Subgradient Variation Type}
\label{DRiSVT}

In this section, we follow the steps in \cref{DR} to study the dynamic regret in subgradient variation type. 
We restore subtraction terms in \cref{OGP-regret-bound} (in \cref{OGP-regret}) and \cref{ONES-regret-bound} (in \cref{ONES-regret}), and assume that $\widehat{x}_t^*$ is the subgradient of the estimated convex loss $\widehat{\varphi}_t$, and $\partial \widehat{\varphi}_t$ is Lipschitz continuous. 
This is equivalent to combining \cref{OGP-corollary} and \cref{ONES-corollary}. 
This section also fixes bugs of \citet{zhao2020dynamic}. 

\subsection*{Step 1: Match of the regret upper bounds} 

Denote by $\boldsymbol{x}_t$ the vector of expert advice and $\widetilde{\boldsymbol{x}}_t$ the vector of all $\widetilde{x}_t$s. Let 
\begin{equation}
\label{VT-DT-hostility}
\begin{aligned}
V_T=4+\varrho^{-2}\sum_{t=1}^{T-1}\sup_{x\in C}\left\lVert x^{\varphi_t}-x^{\widehat{\varphi}_{t}}\right\rVert^2, \quad\text{and}\quad
D_T=L^2 \varrho^{-2}\left(\rho^{2}+\sum_{t=1}^{T-1}\max\left\lVert \boldsymbol{x}_{t}-\widetilde{\boldsymbol{x}}_{t}\right\rVert^2\right).
\end{aligned}
\end{equation}
Note that $V_T$ in \cref{VT-DT-hostility} is different from \cref{gradient-variation}. 
We call $V_T$ in \cref{VT-DT-hostility} the subgradient variation term, which is the general form of gradient variation term. 

According to \cref{OGP-corollary}, the expert who operates OGP with the parameter $\eta$ yields the dynamic regret upper bound as follows, 
\begin{equation*}
\begin{aligned}
\frac{\rho\left(\rho+2 P_T\right)}{2\eta}+\eta\varrho^2\left(V_T+1_{\eta>\frac{1}{\sqrt{2}L}}D_T\right). 
\end{aligned}
\end{equation*}
According to \cref{ONES-corollary}, the static regret upper bound for ONES (meta-algorithm) is 
\begin{equation*}
\begin{aligned}
\frac{-\ln w_1\left(j\right)}{\theta}+\theta\rho^2\varrho^2 V_T, \quad\theta\leqslant\frac{1}{\sqrt{2}\rho^2 L}. 
\end{aligned}
\end{equation*}
The regret upper bounds are matched since 
\begin{equation*}
\begin{aligned}
V_T\leqslant V_T+1_{\eta>\frac{1}{\sqrt{2}L}}D_T.
\end{aligned}
\end{equation*} 

\begin{remark}
Since we do not yet know how to match the subtraction terms introduced by auxiliary strategies, the dynamic regret in subgradient variation type does not involve the auxiliary strategies. 
\end{remark}

\subsection*{Step 2: Allocation of the group of experts} 

If we choose $V_T+D_T$ as the characteristic item, then the global dynamic regret upper bound is $O\big(\sqrt{\left(1+P_T\right)\left(V_T+D_T\right)}\,\big)$, and the corresponding group of experts is $\left\{ e_\lambda\right\}_{\lambda\in \mathcal{E}}$, where 
\begin{equation}
\label{expert-group-1}
\begin{aligned}
\mathcal{E}=\left\{0,1,\cdots,\left\lfloor\log_2\sqrt{2T-1}\right\rfloor\right\},
\end{aligned}
\end{equation} 
the expert $ e_\lambda$ operates OGP with $\eta_{ e_\lambda}=\frac{\rho}{\varrho\sqrt{V_T+D_T}}2^\lambda$. 

If we choose $V_T$ as the characteristic item, then the local dynamic regret is $O\big(\sqrt{\left(1+P_T\right)V_T}\,\big)$, and the corresponding group of experts is $\left\{\epsilon_\mu\right\}_{\mu\in \mathscr{E}}$, where 
\begin{equation}
\label{expert-group-2}
\begin{aligned}
\mathscr{E}=\left\{\mu\in\mathcal{E}\left|\,\frac{\rho}{\varrho\sqrt{V_T}}2^\mu\leqslant\frac{1}{\sqrt{2}L}\right.\right\},
\end{aligned}
\end{equation} 
the expert $\epsilon_\mu$ operates OGP with $\eta_{\epsilon_\mu}=\frac{\rho}{\varrho\sqrt{V_T}}2^\mu$. 

We merge two expert groups and utilize ONES to track the best expert, which is summarized as the following theorem. 

\begin{theorem}
\label{2-expert-group}
Let $V_T$ and $V_T+D_T$ be fixed, and active a set of experts $\left\{ e_\lambda\right\}_{\lambda\in \mathcal{E}}\cup\left\{\epsilon_\mu\right\}_{\mu\in \mathscr{E}}$, where $\mathcal{E}$ and $\mathscr{E}$ follow from \cref{expert-group-1} and \cref{expert-group-2} respectively. 
The expert $e_\lambda$ operates OGP with $\eta_{ e_\lambda}=\frac{\rho}{\varrho\sqrt{V_T+D_T}}2^\lambda$, and the expert $\epsilon_\mu$ operates OGP with $\eta_{\epsilon_\mu}=\frac{\rho}{\varrho\sqrt{V_T}}2^\mu$. If $\theta\propto\frac{1}{\sqrt{V_T}}$ and $\theta\leqslant\frac{1}{\sqrt{2}\rho^2 L}$, where $\theta$ is the parameter of ONES, then we have 
\begin{equation*}
\begin{aligned}
\sum_{t=1}^{T}\varphi_t\left(\overline{x}_t\right)-\varphi_t\left(z_t\right)
<O\left(\sqrt{\left(1+P_T\right)\left(V_T+1_{L^2\rho\left(\rho+2 P_T\right)\leqslant\varrho^2 V_T}D_T\right)}\right).
\end{aligned}
\end{equation*}
\end{theorem}

\begin{remark}
\cref{2-expert-group} shows that two groups of experts are used to track the global bound and the local bound respectively, which is consistent with Ader's idea of ``covering'' the range of $P_T$ with an expert group to hedge the uncertainty of $P_T$. 
However, \cite{zhao2020dynamic} only use one group of experts to track the local bound, resulting in the range of $P_T$ cannot be ``covered'', thus unable to hedge the uncertainty of $P_T$. 
Therefore, their upper bound is an affine function of $P_T$, which is consistent with \cite{zinkevich2003online}. 
\cref{2-expert-group} states that $T$ is successfully replaced by $V_T+1_{L^2\rho\left(\rho+2 P_T\right)\leqslant\varrho^2 V_T}D_T$, which fixes bugs of \cite{zhao2020dynamic}. 
The upper bound is an affine function of $\sqrt{P_T}$, which is consistent with~\cite{zhang2018adaptive}. Moreover, if $V_T+1_{L^2\rho\left(\rho+2 P_T\right)\leqslant\varrho^2 V_T}D_T$ grows sublinearly, then the sublinear dynamic regret holds for arbitrary reference strategy sequence. 
\end{remark}

We call the above algorithm subgradient variation version of ONES-OGP. 

\subsection*{Step 3: Adaptive Trick}

Similar to \cref{adaptive-trick}, we utilize the following adaptive trick to unfreeze $V_T$ and $V_T+D_T$. 

\begin{theorem}[Adaptive Trick]
\label{adaptive-trick-2}
The adaptive trick 
\begin{itemize}
\item[] calls subgradient variation version of ONES-OGP with $\theta\propto 2^{-m'}$, $\theta\leqslant\frac{1}{\sqrt{2}\rho^2 L}$, and $\eta_{e_\lambda}=\frac{\rho}{\varrho}2^{\lambda-m}$ for $\lambda=0, 1, \cdots, n$, $\eta_{\epsilon_\mu}=\frac{\rho}{\varrho}2^{\mu-m'}$ for $\mu=0, 1, \cdots, \left\lvert\mathscr{E}\right\rvert-1$, 
\item[] under the constraints that $V_T+D_T\in\big[4^m, 4^{m+1}\big)$, $V_T\in\big[4^{m'}, 4^{m'+1}\big)$ and $T\in\frac{1}{2}\big[4^n, 4^{n+1}\big)+1$. 
\end{itemize}
The above execution process achieves the following dynamic regret upper bound, 
\begin{equation*}
O\left(\sqrt {\left(1+P_T\right)\left(V_T+1_{L^2\rho\left(\rho+2 P_T\right)\leqslant\varrho^2 V_T}D_T\right)}\right).
\end{equation*}
\end{theorem}

\begin{remark}
Let $m'$ indicates the stage index of the game, then the proof of \cref{adaptive-trick-2} is similar to the proof of \cref{adaptive-trick} (See \cref{pf:adaptive-trick}). 
\end{remark}

To make it easier to follow, we depict the above specific execution process in \cref{alg-4}. 

\begin{figure}[!ht]
\begin{algorithm}[H]
\caption{Subgradient variation version of ONES-OGP with adaptive trick}
\label{alg-4}
\begin{algorithmic}[1]
\STATE $m\gets -1$, \,$m'\gets -1$, \,$n\gets -1$, \,$n'\gets -1$
\FOR{round $t=1, 2, \cdots$}
\STATE $n\gets\left\lvert\mathcal{E}\right\rvert-1$,\, $n'\gets\left\lvert\mathscr{E}\right\rvert-1$,\, $m\gets\left\lfloor\log_4 \left(V_t+D_t\right)\right\rfloor$,\, $m'\gets\left\lfloor\log_4 V_t\right\rfloor$, where $V_t$ and $D_t$ are calculated by \cref{VT-DT-hostility}
\IF{$\left(m\text{ or } m'\text{ or } n \text{ or } n'\right)$ changed}
\STATE Construct a set of experts $\left\{ e_\lambda\right\}_{\lambda\in \mathcal{E}}\cup\left\{\epsilon_\mu\right\}_{\mu\in \mathscr{E}}$ and 
invoke \cref{alg-6} with $\eta_{e_\lambda}=\frac{\rho}{\varrho}2^{\lambda-m}$ for $ e_\lambda$, 
invoke \cref{alg-6} with $\eta_{\epsilon_\mu}=\frac{\rho}{\varrho}2^{\mu-m'}$ for $\epsilon_\mu$ if $\mathscr{E}\neq\varnothing$
\STATE Call \cref{alg-5} with parameter $n$, $n'$ and $\theta\propto 2^{-m'}$, where $\theta\leqslant\frac{1}{\sqrt{2}\rho^2 L}$
\ENDIF
\STATE Receive the estimated convex loss $\widehat{\varphi}_t$ from an arbitrary estimating process with $\partial\widehat{\varphi}_t$ to be Lipschitz continuous, send $\partial\widehat{\varphi}_{t}$ to \cref{alg-5} and each expert
\STATE Call \cref{alg-5} to get expert advice strategies $\boldsymbol{x}_t$, $\widetilde{\boldsymbol{x}}_t$, and the weight $w_t$
\STATE Output strategy $\overline{x}_t=\left\langle w_t, \boldsymbol{x}_t\right\rangle$, and then observe loss function $\varphi_t$
\STATE Send $\ell_t\in\left\langle\partial\varphi_t\left(\overline{x}_t\right), \boldsymbol{x}_t\right\rangle$ to \cref{alg-5}, send $\partial\varphi_t$ to each expert
\ENDFOR
\end{algorithmic}
\end{algorithm}
\begin{algorithm}[H]
\caption{Subprogram: ONES with parameter $n$, $n'$ and $\theta$}
\label{alg-5}
\begin{algorithmic}[1]
\REQUIRE $\ell_\tau$ and $\partial\widehat{\varphi}_{\tau+1}$ from \cref{alg-4}
\ENSURE $w_{\tau+1}\left( e_\lambda\right)$, $\lambda=0, 1, \cdots, n$, and $w_{\tau+1}\left( e_\mu\right)$, $\mu=0, 1, \cdots, n'$
\STATE $\widetilde{w}_1\left( e_\lambda\right)=\beta\left(\lambda+2\right)^{-\alpha}$, $\lambda=0, 1, \cdots, n$, \,$\widetilde{w}_1\left(\epsilon_\mu\right)=\beta\left(\mu+2\right)^{-\alpha}$, $\mu=0, 1, \cdots, n'$
\STATE Get expert advice strategies $\boldsymbol{x}_\tau$ and $\widetilde{\boldsymbol{x}}_\tau$, send them to \cref{alg-4}
\STATE Each call follows the following rule
\begin{equation*}
\begin{aligned}
\widetilde{w}_{\tau+1}&=\mathscr{N}\big(\widetilde{w}_{\tau}\circ\mathrm{e}^{-\theta\ell_{\tau}}\big), \\
\widehat{\ell}_{\tau+1}&\in\left\langle\partial\widehat{\varphi}_{\tau+1}\left(\left\langle \widetilde{w}_{\tau+1}, \boldsymbol{x}_\tau\right\rangle\right), \boldsymbol{x}_\tau\right\rangle,\\
w_{\tau+1}&=\mathscr{N}\big(\widetilde{w}_{\tau+1}\circ\mathrm{e}^{- \theta\widehat{\ell}_{\tau+1}}\big)
\end{aligned}
\end{equation*}
\end{algorithmic}
\end{algorithm}
\begin{algorithm}[H]
\caption{Subprogram: OGP with parameter $\eta$}
\label{alg-6}
\begin{algorithmic}[1]
\REQUIRE $\partial\varphi_\tau$ and $\partial\widehat{\varphi}_{\tau+1}$ from \cref{alg-4}
\ENSURE Each call follows the OGP (Equation~\ref{OGP-iteration})
\end{algorithmic}
\end{algorithm}
\end{figure}

\section{Comparisons}

From \cref{DR,DRwAS,DRiSVT}, we replace the dependence of the dynamic regret upper bound on $T$ with $M_T$, $\widetilde{M}_T$ and $V_T+1_{L^2\rho\left(\rho+2 P_T\right)\leqslant\varrho^2 V_T}D_T$ respectively. 
All these characteristic terms are $O\left(T\right)$ in the worst case while be much smaller in benign environments. 

Both $M_T$ and $\widetilde{M}_T$ are measures of estimation accuracy. 
The meaning is intuitive, that is, the higher the prediction accuracy, the slower the growth of $M_T$ (or $\widetilde{M}_T$). 
Moreover, $\widetilde{M}_T$ is a tighter measure than $M_T$ due to the introduction of auxiliary strategies that lead to tighter regret upper bounds. 

The characteristic item $V_T+1_{L^2\rho\left(\rho+2 P_T\right)\leqslant\varrho^2 V_T}D_T$ is proposed to fix bugs of \citet{zhao2020dynamic}, where $V_T$ is the subgradient variation term, the general form of gradient variation term. 
Compared with $M_T$ and $\widetilde{M}_T$ obtained by dropping the subtraction term in regret, $V_T+1_{L^2\rho\left(\rho+2 P_T\right)\leqslant\varrho^2 V_T}D_T$ relies on the subtraction term. 
However, this does not make $V_T+1_{L^2\rho\left(\rho+2 P_T\right)\leqslant\varrho^2 V_T}D_T$ tighter than $M_T$ or $\widetilde{M}_T$. 
Indeed, when the predicted loss equals to the true loss, i.e., $\widehat{\varphi}_{t}=\varphi_{t}$, $V_T+1_{L^2\rho\left(\rho+2 P_T\right)\leqslant\varrho^2 V_T}D_T$ may still grow, while if all predictions are accurate, both $M_T$ and $\widetilde{M}_T$ stop growing. 

It is worth mentioning that $M_T$, $\widetilde{M}_T$ and $V_T+1_{L^2\rho\left(\rho+2 P_T\right)\leqslant\varrho^2 V_T}D_T$ are all characteristic terms induced by optimism. 
Furthermore, we argue that optimism is the driving force behind, and focusing too much on the regret upper bounds of certain characteristic terms, such as the gradient variation term, may deviate from the essence of the online learning problem. 

\section{Conclusions and Future Work}

In this paper, we study the optimistic online convex optimization problem in dynamic environments. 
We follow the idea of Ader, replace GP and NES in Ader with OGP and ONES respectively, extend the doubling trick to the adaptive trick, replace the dependence of the dynamic regret on $T$ with $M_T$, $\widetilde{M}_T$ or $V_T+1_{L^2\rho\left(\rho+2 P_T\right)\leqslant\varrho^2 V_T}D_T$,and obtain environment-adaptive algorithms. 

Optimism may be the hub to linking online learning theories. 
For a non-optimistic algorithm, it suffices to set the estimated loss to be null, and for learning with delay, it suffices to modify the estimated loss to delete the unobserved loss subgradients \citep{flaspohler2021online}. 
This paper further studies the role of optimism in the framework of environment-adaptive algorithms. 
We hope that this work encourages in-depth research on the unified theory of online learning with dynamic regret as the performance metric and optimism as the core idea. 


\clearpage
\bibliography{reference}

\newpage
\appendix
\onecolumn

\section{Proof of \cref{OGP-regret}}
\label{proof-OGP-regret}

The proof of \cref{OGP-regret} relies on the following lemma. 
Part of the proof is inspired by \citet{zhao2020dynamic}. 

\begin{lemma}[Theorem~5.2 of {\citealp{brezis2010functional}}]
\label{best-approx}
Let $H$ be a Hilbert space, and let $C\subset H$ be a nonempty closed convex set. Then $\forall x\in H$, $\exists ! x_0=P_C\left(x\right)$, such that 
$\left\langle C-x_0, x-x_0\right\rangle\leqslant 0$. 
\end{lemma}

\begin{proof}
We rearrange OGP as follows, 
\begin{equation*}
\begin{aligned}
\widetilde{y}_{t+1}&=\widetilde{x}_t-\eta x_t^*, && \widetilde{x}_{t+1}=P_C\left(\widetilde{y}_{t+1}\right),\\
y_{t+1}&=\widetilde{x}_{t+1}-\eta\widehat{x}_{t+1}^*,&& x_{t+1}=P_C\left(y_{t+1}\right).
\end{aligned}
\end{equation*}
Note that 
\begin{equation*}
\varphi_t\left(x_t\right)-\varphi_t\left(z_t\right)
\leqslant\frac{1}{\eta}\left\langle \eta x_t^*,x_t-z_t\right\rangle, \quad x_t^*\in\partial\varphi_t\left(x_t\right),
\end{equation*}
and
\begin{equation*}
\begin{aligned}
\left\langle \eta x_t^*,x_t-z_t\right\rangle
&=\eta \left\langle x_t^*-\widehat{x}_t^*,x_t-\widetilde{x}_{t+1}\right\rangle+\left\langle \eta x_t^*,\widetilde{x}_{t+1}-z_t\right\rangle+\left\langle \eta\widehat{x}_{t}^*,x_t-\widetilde{x}_{t+1}\right\rangle \\
&=\eta \left\langle x_t^*-\widehat{x}_t^*,x_t-\widetilde{x}_{t+1}\right\rangle-\left\langle \widetilde{x}_t-\widetilde{y}_{t+1},z_t-\widetilde{x}_{t+1}\right\rangle-\left\langle \widetilde{x}_{t}-y_{t},\widetilde{x}_{t+1}-x_t\right\rangle,
\end{aligned}
\end{equation*}
where
\begin{equation*}
\begin{aligned}
\eta \left\langle x_t^*-\widehat{x}_t^*,x_t-\widetilde{x}_{t+1}\right\rangle
\leqslant\eta \left\lVert x_t^*-\widehat{x}_t^*\right\rVert\left\lVert x_t-\widetilde{x}_{t+1}\right\rVert
\leqslant \frac{\eta^2}{2}\left\lVert x_t^*-\widehat{x}_t^*\right\rVert^2+\frac{1}{2}\left\lVert x_t-\widetilde{x}_{t+1}\right\rVert^2,
\end{aligned}
\end{equation*}
and
\begin{equation*}
\begin{aligned}
\frac{1}{2}\left\lVert x_t-\widetilde{x}_{t+1}\right\rVert^2
\leqslant\frac{1}{2}\left\lVert\widetilde{x}_{t+1}\right\rVert^2-\frac{1}{2}\left\lVert x_t\right\rVert^2+\left\langle y_t,x_t-\widetilde{x}_{t+1}\right\rangle,
\end{aligned}
\end{equation*}
since $\left\langle\widetilde{x}_{t+1}-x_t,y_t-x_t\right\rangle\leqslant 0$ holds according to \cref{best-approx}. Thus
\begin{equation*}
\begin{aligned}
\left\langle \eta x_t^*,x_t-z_t\right\rangle
&\leqslant\frac{\eta^2}{2}\left\lVert x_t^*-\widehat{x}_t^*\right\rVert^2+\frac{1}{2}\left\lVert\widetilde{x}_{t+1}\right\rVert^2-\frac{1}{2}\left\lVert x_t\right\rVert^2
-\left\langle \widetilde{x}_t-\widetilde{y}_{t+1},z_t-\widetilde{x}_{t+1}\right\rangle-\left\langle \widetilde{x}_{t},\widetilde{x}_{t+1}-x_t\right\rangle \\
&\leqslant\frac{\eta^2}{2}\left\lVert x_t^*-\widehat{x}_t^*\right\rVert^2+\frac{1}{2}\left\lVert z_t-\widetilde{x}_t\right\rVert^2-\frac{1}{2}\left\lVert z_t-\widetilde{x}_{t+1}\right\rVert^2-\frac{1}{2}\left\lVert x_{t}-\widetilde{x}_{t}\right\rVert^2,
\end{aligned}
\end{equation*}
since $\left\langle z_t-\widetilde{x}_{t+1}, \widetilde{y}_{t+1}-\widetilde{x}_{t+1}\right\rangle\leqslant 0$ according to \cref{best-approx}. 
So we have 
\begin{equation*}
\begin{aligned}
&\sum_{t=1}^{T}\varphi_t\left(x_t\right)-\varphi_t\left(z_t\right)\\
\leqslant\ &\frac{1}{2\eta}\sum_{t=1}^{T}\left(\left\lVert z_t-\widetilde{x}_t\right\rVert^2-\left\lVert z_t-\widetilde{x}_{t+1}\right\rVert^2\right)+\frac{\eta}{2}\sum_{t=1}^{T}\left\lVert x_t^*-\widehat{x}_t^*\right\rVert^2-\frac{1}{2\eta}\sum_{t=1}^{T}\left\lVert x_{t}-\widetilde{x}_{t}\right\rVert^2 \\
\leqslant\ &\frac{1}{2\eta}\left\lVert z_1-\widetilde{x}_1\right\rVert^2+\frac{1}{\eta}\sum_{t=2}^{T}\left\lVert \frac{z_t+z_{t-1}}{2}-\widetilde{x}_{t}\right\rVert\left\lVert z_t-z_{t-1}\right\rVert+\frac{\eta}{2}\sum_{t=1}^{T}\left\lVert x_t^*-\widehat{x}_t^*\right\rVert^2-\frac{1}{2\eta}\sum_{t=1}^{T}\left\lVert x_{t}-\widetilde{x}_{t}\right\rVert^2 \\
\leqslant\ &\frac{\rho^2}{2\eta}+\frac{\rho}{\eta}\sum_{t=2}^{T}\left\lVert z_t-z_{t-1}\right\rVert+\frac{\eta}{2}\sum_{t=1}^{T}\left\lVert x_t^*-\widehat{x}_t^*\right\rVert^2-\frac{1}{2\eta}\sum_{t=1}^{T}\left\lVert x_{t}-\widetilde{x}_{t}\right\rVert^2. 
\end{aligned}
\end{equation*}
\end{proof}

\section{Proof of \cref{auxiliary-OGP-regret}}
\label{pf:auxiliary-OGP-regret}

The proof of \cref{auxiliary-OGP-regret} relies on the following lemma. 
The proof process follows the idea of Appendix~B of \citet{flaspohler2021online}. 

\begin{lemma}[Lemma~15 of {\citealp{flaspohler2021online}}]
\label{auxiliary-neq}
$\left\lVert x_t-y_t\right\rVert\leqslant\eta\left\lVert \widehat{x}_{t}^*-\widehat{y}_{t}^*\right\rVert$, 
where $x_t$ and $\widehat{x}_{t}^*$ are determined by OGP (Equation~\ref{OGP-iteration}), and $y_t$ and $\widehat{y}_{t}^*$ are determined by the following auxiliary workflow, 
\begin{equation}
\label{auxiliary-OGP-iteration}
\begin{aligned}
\widetilde{y}_{t+1}&=P_C\left(\widetilde{y}_t-\eta x_t^*\right), && \widetilde{y}_1\in C,\\
y_{t+1}&=P_C\left(\widetilde{y}_{t+1}-\eta\widehat{y}_{t+1}^*\right),&& \widehat{y}_t^*=\lambda\widehat{x}_t^*+\left(1-\lambda\right)x_t^*,\quad\lambda=\min\left\{\left\lVert x_t^*\right\rVert \big/ \left\lVert x_t^*-\widehat{x}_t^*\right\rVert,\,1\right\}. 
\end{aligned}
\end{equation}
\end{lemma}

\begin{proof}
Choose the auxiliary workflow as \cref{auxiliary-OGP-iteration}. 
The dynamic regret can be decomposed as the following form, 
\begin{equation*}
\begin{aligned}
\sum_{t=1}^T\varphi_t\left(x_t\right)-\varphi_t\left(z_t\right)
\leqslant\sum_{t=1}^T\left\langle x_t^*,x_t-z_t\right\rangle
=\sum_{t=1}^T\left\langle x_t^*,x_t-y_t\right\rangle+\sum_{t=1}^T\left\langle x_t^*,y_t-z_t\right\rangle,
\end{aligned}
\end{equation*}
where $\left\langle x_t^*,x_t-y_t\right\rangle\leqslant\left\lVert x_t^*\right\rVert\left\lVert x_t-y_t\right\rVert\leqslant\eta\left\lVert x_t^*\right\rVert\left\lVert \widehat{x}_{t}^*-\widehat{y}_{t}^*\right\rVert$ according to \cref{auxiliary-neq}, and 
\begin{equation*}
\begin{aligned}
\sum_{t=1}^T\left\langle x_t^*,y_t-z_t\right\rangle
\leqslant\frac{\rho^2}{2\eta}+\frac{\rho}{\eta}\sum_{t=2}^{T}\left\lVert z_t-z_{t-1}\right\rVert+\frac{\eta}{2}\sum_{t=1}^{T}\left\lVert x_t^*-\widehat{y}_t^*\right\rVert^2
\end{aligned}
\end{equation*}
according to \cref{OGP-regret}. 
To complete the proof, it suffices to note that 
\begin{equation*}
\begin{aligned}
\left\lVert x_t^*-\widehat{y}_t^*\right\rVert^2+2\left\lVert x_t^*\right\rVert\left\lVert \widehat{x}_{t}^*-\widehat{y}_{t}^*\right\rVert
&=\lambda^2\left\lVert x_t^*-\widehat{x}_t^*\right\rVert^2+2\left(1-\lambda\right)\left\lVert x_t^*\right\rVert\left\lVert x_{t}^*-\widehat{x}_{t}^*\right\rVert \\
&=h_{\left\lVert x_t^*\right\rVert}\left(\left\lVert x_t^*-\widehat{x}_t^*\right\rVert\right).
\end{aligned}
\end{equation*}
\end{proof}

\section{Proof of \cref{OGP-corollary}}

\begin{proof}
Let $x_t^{\varphi_t}=x_t^*$, $\widetilde{x}_t^{\widehat{\varphi}_{t}}=\widehat{x}_t^*$. 
Note that 
\begin{equation*}
\begin{aligned}
\left\lVert x_t^{\varphi_t}-\widetilde{x}_{t}^{\widehat{\varphi}_{t}}\right\rVert^2
\leqslant \left(\left\lVert x_t^{\varphi_t}-x_t^{\widehat{\varphi}_t}\right\rVert + \left\lVert x_t^{\widehat{\varphi}_t}-\widetilde{x}_{t}^{\widehat{\varphi}_{t}}\right\rVert\right)^2
\leqslant 2\left\lVert x_t^{\varphi_t}-x_{t}^{\widehat{\varphi}_{t}}\right\rVert^2 + 2L^2\left\lVert x_t-\widetilde{x}_t\right\rVert^2,
\end{aligned}
\end{equation*}
where $x_t^{\widehat{\varphi}_t}\in\partial\widehat{\varphi}_t\left(x_{t}\right)$. According to \cref{OGP-regret}, the dynamic regret upper bound for OGP is 
\begin{equation*}
\begin{aligned}
&\frac{\rho^2}{2\eta}+\frac{\rho}{\eta}P_T+\eta\sum_{t=1}^{T}\left\lVert x_t^{\varphi_t}-x_{t}^{\widehat{\varphi}_{t}}\right\rVert^2+\left(\eta L^2-\frac{1}{2\eta}\right)\sum_{t=1}^{T}\left\lVert x_{t}-\widetilde{x}_{t}\right\rVert^2  \\
\leqslant\ &\frac{\rho\left(\rho+2 P_T\right)}{2\eta}+\eta\sum_{t=1}^{T}\left(\sup_{x\in C}\left\lVert x^{\varphi_t}-x^{\widehat{\varphi}_{t}}\right\rVert^2+1_{\eta>\frac{1}{\sqrt{2}L}}L^2 \left\lVert x_{t}-\widetilde{x}_{t}\right\rVert^2\right). 
\end{aligned}
\end{equation*}
\end{proof}

\section{A Counterexample}
\label{Counterexample}

This section illustrates a counterexample. 
Real losses are not Lipschitz continuous, but estimated losses must be $L$-Lipschitz continuous. We claim that the subgradient variation term $\sum_{t=1 }^T\sup_{x\in C}\left\lVert x^{\varphi_t}-x^{\widehat{\varphi}_{t}}\right\rVert^2$ may converge. 

Consider two monotone multivalued function sequences $\left\{\partial\varphi_t\right\}_{t=1}^T$ and $\left\{\partial\widehat{\varphi}_t\right\}_{t=1}^T$ defined on $C=\left[-1,1\right]$. The trend of $\partial\varphi_t$s' graphs is as follows. 

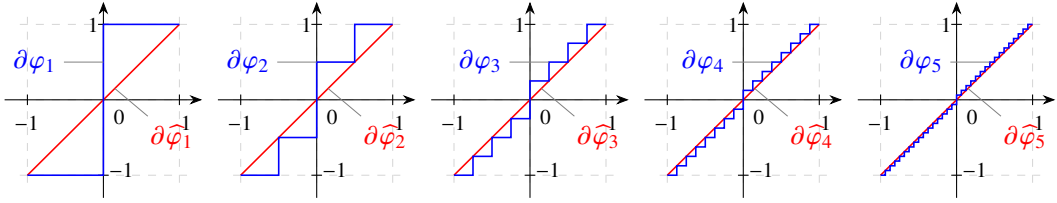
\begin{center}
\tikz{
\draw[help lines, color=gray!33, dashed] (-1.1,-1.1) grid (1.1,1.1);
\draw[arrows={-Stealth}] (-1.3,0) -- (1.3,0);
\draw[arrows={-Stealth}] (0,-1.3) -- (0,1.3);
\draw[shift={(0,0)}] (0pt,0pt) -- (0pt,0pt) node[below right] {$^0$};
\draw[shift={(1,0)}] (0pt,2pt) -- (0pt,-2pt) node[below] {${\ \ }^1$};
\draw[shift={(-1,0)}] (0pt,2pt) -- (0pt,-2pt) node[below] {$^{-1}$};
\draw[shift={(0,1)}] (-2pt,0pt) -- (2pt,0pt) node[left] {\scriptsize $1$};
\draw[shift={(0,-1)}] (-2pt,0pt) -- (2pt,0pt) node[right] {${\!\!}^{-1}$};
\datavisualization
[xy Cartesian,
visualize as line=stp,
visualize as line=lin,
stp={style=blue},
lin={style=red},
stp={pin in data={text=\color{blue}{$\partial\varphi_1$}, when=y is 0.5}},
lin={pin in data={text'=\color{red}{$\partial\widehat{\varphi}_{1}$}, when=x is 0.15}}]
data [set=stp] {
x,y
-1,-1
0,-1
0,0.5
0,1
1,1
}
data [set=lin] {
x,y
-1,-1
0.15,0.15
1,1
};}
\ 
\tikz{
\draw[help lines, color=gray!33, dashed] (-1.1,-1.1) grid (1.1,1.1);
\draw[arrows={-Stealth}] (-1.3,0) -- (1.3,0);
\draw[arrows={-Stealth}] (0,-1.3) -- (0,1.3);
\draw[shift={(0,0)}] (0pt,0pt) -- (0pt,0pt) node[below right] {$^0$};
\draw[shift={(1,0)}] (0pt,2pt) -- (0pt,-2pt) node[below] {${\ \ }^1$};
\draw[shift={(-1,0)}] (0pt,2pt) -- (0pt,-2pt) node[below] {$^{-1}$};
\draw[shift={(0,1)}] (-2pt,0pt) -- (2pt,0pt) node[left] {\scriptsize $1$};
\draw[shift={(0,-1)}] (-2pt,0pt) -- (2pt,0pt) node[right] {${\!\!}^{-1}$};
\datavisualization
[xy Cartesian, 
visualize as line=stp,
visualize as line=lin,
stp={style=blue},
lin={style=red},
stp={pin in data={text=\color{blue}{$\partial\varphi_2$}, when=y is 0.5}},
lin={pin in data={text'=\color{red}{$\partial\widehat{\varphi}_{2}$}, when=x is 0.15}}]
data [set=stp] {
x,y
-1,-1
-0.5,-1
-0.5,-0.5
0,-0.5
0,0.5
0.5,0.5
0.5,1
1,1
}
data [set=lin] {
x,y
-1,-1
0.15,0.15
1,1
};}
\ 
\tikz{
\draw[help lines, color=gray!33, dashed] (-1.1,-1.1) grid (1.1,1.1);
\draw[arrows={-Stealth}] (-1.3,0) -- (1.3,0);
\draw[arrows={-Stealth}] (0,-1.3) -- (0,1.3);
\draw[shift={(0,0)}] (0pt,0pt) -- (0pt,0pt) node[below right] {$^0$};
\draw[shift={(1,0)}] (0pt,2pt) -- (0pt,-2pt) node[below] {${\ \ }^1$};
\draw[shift={(-1,0)}] (0pt,2pt) -- (0pt,-2pt) node[below] {$^{-1}$};
\draw[shift={(0,1)}] (-2pt,0pt) -- (2pt,0pt) node[left] {\scriptsize $1$};
\draw[shift={(0,-1)}] (-2pt,0pt) -- (2pt,0pt) node[right] {${\!\!}^{-1}$};
\datavisualization
[xy Cartesian,
visualize as line=stp,
visualize as line=lin,
stp={style=blue},
lin={style=red},
stp={pin in data={text=\color{blue}{$\partial\varphi_3$}, when=y is 0.5}},
lin={pin in data={text'=\color{red}{$\partial\widehat{\varphi}_{3}$}, when=x is 0.15}}]
data [set=stp] {
x,y
-1,-1
-0.75,-1
-0.75,-0.75
-0.5,-0.75
-0.5,-0.5
-0.25,-0.5
-0.25,-0.25
0,-0.25
0,0.25
0.25,0.25
0.25,0.5
0.5,0.5
0.5,0.75
0.75,0.75
0.75,1
1,1
}
data [set=lin] {
x,y
-1,-1
0.15,0.15
1,1
};}
\ 
\tikz{
\draw[help lines, color=gray!33, dashed] (-1.1,-1.1) grid (1.1,1.1);
\draw[arrows={-Stealth}] (-1.3,0) -- (1.3,0);
\draw[arrows={-Stealth}] (0,-1.3) -- (0,1.3);
\draw[shift={(0,0)}] (0pt,0pt) -- (0pt,0pt) node[below right] {$^0$};
\draw[shift={(1,0)}] (0pt,2pt) -- (0pt,-2pt) node[below] {${\ \ }^1$};
\draw[shift={(-1,0)}] (0pt,2pt) -- (0pt,-2pt) node[below] {$^{-1}$};
\draw[shift={(0,1)}] (-2pt,0pt) -- (2pt,0pt) node[left] {\scriptsize $1$};
\draw[shift={(0,-1)}] (-2pt,0pt) -- (2pt,0pt) node[right] {${\!\!}^{-1}$};
\datavisualization
[xy Cartesian, 
visualize as line=stp,
visualize as line=lin,
stp={style=blue},
lin={style=red},
stp={pin in data={text=\color{blue}{$\partial\varphi_4$}, when=y is 0.5}},
lin={pin in data={text'=\color{red}{$\partial\widehat{\varphi}_{4}$}, when=x is 0.15}}]
data [set=stp] {
x,y
-1,-1
-0.875,-1
-0.875,-0.875
-0.75,-0.875
-0.75,-0.75
-0.625,-0.75
-0.625,-0.625
-0.5,-0.625
-0.5,-0.5
-0.375,-0.5
-0.375,-0.375
-0.25,-0.375
-0.25,-0.25
-0.125,-0.25
-0.125,-0.125
0,-0.125
0,0.125
0.125,0.125
0.125,0.25
0.25,0.25
0.25,0.375
0.375,0.375
0.375,0.5
0.5,0.5
0.5,0.625
0.625,0.625
0.625,0.75
0.75,0.75
0.75,0.875
0.875,0.875
0.875,1
1,1
}
data [set=lin] {
x,y
-1,-1
0.15,0.15
1,1
};}
\ 
\tikz{
\draw[help lines, color=gray!33, dashed] (-1.1,-1.1) grid (1.1,1.1);
\draw[arrows={-Stealth}] (-1.3,0) -- (1.3,0);
\draw[arrows={-Stealth}] (0,-1.3) -- (0,1.3);
\draw[shift={(0,0)}] (0pt,0pt) -- (0pt,0pt) node[below right] {$^0$};
\draw[shift={(1,0)}] (0pt,2pt) -- (0pt,-2pt) node[below] {${\ \ }^1$};
\draw[shift={(-1,0)}] (0pt,2pt) -- (0pt,-2pt) node[below] {$^{-1}$};
\draw[shift={(0,1)}] (-2pt,0pt) -- (2pt,0pt) node[left] {\scriptsize $1$};
\draw[shift={(0,-1)}] (-2pt,0pt) -- (2pt,0pt) node[right] {${\!\!}^{-1}$};
\datavisualization
[xy Cartesian, 
visualize as line=stp,
visualize as line=lin,
stp={style=blue},
lin={style=red},
stp={pin in data={text=\color{blue}{$\partial\varphi_5$}, when=y is 0.5}},
lin={pin in data={text'=\color{red}{$\partial\widehat{\varphi}_{5}$}, when=x is 0.15}}]
data [set=stp] {
x,y
-1,-1
-0.9375,-1
-0.9375,-0.9375
-0.875,-0.9375
-0.875,-0.875
-0.8125,-0.875
-0.8125,-0.8125
-0.75,-0.8125
-0.75,-0.75
-0.6875,-0.75
-0.6875,-0.6875
-0.625,-0.6875
-0.625,-0.625
-0.5625,-0.625
-0.5625,-0.5625
-0.5,-0.5625
-0.5,-0.5
-0.4375,-0.5
-0.4375,-0.4375
-0.375,-0.4375
-0.375,-0.375
-0.3125,-0.375
-0.3125,-0.3125
-0.25,-0.3125
-0.25,-0.25
-0.1875,-0.25
-0.1875,-0.1875
-0.125,-0.1875
-0.125,-0.125
-0.0625,-0.125
-0.0625,-0.0625
0,-0.0625
0,0.0625
0.0625,0.0625
0.0625,0.125
0.125,0.125
0.125,0.1875
0.1875,0.1875
0.1875,0.25
0.25,0.25
0.25,0.3125
0.3125,0.3125
0.3125,0.375
0.375,0.375
0.375,0.4375
0.4375,0.4375
0.4375,0.5
0.5,0.5
0.5,0.5625
0.5625,0.5625
0.5625,0.625
0.625,0.625
0.625,0.6875
0.6875,0.6875
0.6875,0.75
0.75,0.75
0.75,0.8125
0.8125,0.8125
0.8125,0.875
0.875,0.875
0.875,0.9375
0.9375,0.9375
0.9375,1
1,1
}
data [set=lin] {
x,y
-1,-1
0.15,0.15
1,1
};}
\end{center}

$\forall t$, $\partial\varphi_t$ is not Lipschitz continuous, $\partial\widehat{\varphi}_t\equiv\id_C$ is $1$-Lipschitz continuous, and \[\sum_{t=1 }^T\sup_{x\in C}\left\lVert x^{\varphi_t}-x^{\widehat{\varphi}_{t}}\right\rVert^2<\sum_{t=1}^{\infty}t^{-2}=\frac{\pi^2}{6}.\] 

\section{Proof of \cref{ONES-regret}}
\label{proof-ONES-regret}

The proof of \cref{ONES-regret} relies on the following lemma. 

\begin{lemma}[Example 2.5 of {\citealp{shwartz2012online}}]
\label{-Entropy-strongly-convex}
$\sum_{i}w\left(i\right)\ln w\left(i\right)$ is $1$-strongly-convex w.r.t $\left\lVert\cdot\right\rVert_1$ over the probability simplex.
\end{lemma}

\begin{proof}
We rearrange ONES as follows, 
\begin{equation*}
\begin{aligned}
\widetilde{v}_{t+1}&=\widetilde{v}_{t}-\theta\ell_{t}, &&
\widetilde{w}_{t+1}=\widetilde{N}_{t+1}\mathrm{e}^{\widetilde{v}_{t+1}},\\
v_{t+1}&=\widetilde{v}_{t+1}-\theta\widehat{\ell}_{t+1}, &&
w_{t+1}=N_{t+1}\mathrm{e}^{v_{t+1}}, 
\end{aligned}
\end{equation*}
where $\widetilde{N}_{t+1}$ and $N_{t+1}$ represent the normalization coefficients.
Note that 
\begin{equation*}
\begin{aligned}
\left\langle \theta \ell_t,w_t-w\right\rangle
&=\theta \left\langle \ell_t-\widehat{\ell}_t,w_t-\widetilde{w}_{t+1}\right\rangle+\left\langle \theta \ell_t,\widetilde{w}_{t+1}-w\right\rangle+\left\langle \theta\widehat{\ell}_{t},w_t-\widetilde{w}_{t+1}\right\rangle \\
&=\theta \left\langle \ell_t-\widehat{\ell}_t,w_t-\widetilde{w}_{t+1}\right\rangle+\left\langle \widetilde{v}_{t}-\widetilde{v}_{t+1},\widetilde{w}_{t+1}-w\right\rangle+\left\langle \widetilde{v}_{t}-v_{t},w_t-\widetilde{w}_{t+1}\right\rangle,
\end{aligned}
\end{equation*}
where
\begin{equation*}
\begin{aligned}
\theta \left\langle \ell_t-\widehat{\ell}_t,w_t-\widetilde{w}_{t+1}\right\rangle
\leqslant\theta \left\lVert \ell_t-\widehat{\ell}_t\right\rVert_\infty\left\lVert w_t-\widetilde{w}_{t+1}\right\rVert_1
\leqslant \frac{\theta^2}{2}\left\lVert \ell_t-\widehat{\ell}_t\right\rVert_\infty^2+\frac{1}{2}\left\lVert w_t-\widetilde{w}_{t+1}\right\rVert_1^2,
\end{aligned}
\end{equation*}
and
\begin{equation*}
\begin{aligned}
\frac{1}{2}\left\lVert w_t-\widetilde{w}_{t+1}\right\rVert_1^2
\leqslant\left\langle \widetilde{w}_{t+1}, \ln \frac{\widetilde{w}_{t+1}}{w_t}\right\rangle,
\end{aligned}
\end{equation*}
since $\left\langle w, \ln w\right\rangle$ is $1$-strongly-convex w.r.t $\left\lVert\cdot\right\rVert_1$ over the probability simplex according to \cref{-Entropy-strongly-convex}. 
Note that 
\begin{equation*}
\begin{aligned}
\left\langle \widetilde{v}_{t}-\widetilde{v}_{t+1},\widetilde{w}_{t+1}-w\right\rangle&+\left\langle \widetilde{v}_{t}-v_{t},w_t-\widetilde{w}_{t+1}\right\rangle \\
&=\left\langle \widetilde{v}_{t}, w_t-w\right\rangle
-\left\langle \widetilde{v}_{t+1},\widetilde{w}_{t+1}-w\right\rangle
-\left\langle v_{t},w_t-\widetilde{w}_{t+1}\right\rangle \\
&=\left\langle \ln\frac{\widetilde{w}_{t}}{\widetilde{N}_{t}}, w_t-w\right\rangle
-\left\langle \ln\frac{\widetilde{w}_{t+1}}{\widetilde{N}_{t+1}},\widetilde{w}_{t+1}-w\right\rangle
-\left\langle \ln\frac{w_{t}}{N_{t}},w_t-\widetilde{w}_{t+1}\right\rangle \\
&=\left\langle \ln\widetilde{w}_{t}, w_t-w\right\rangle
-\left\langle \ln\widetilde{w}_{t+1},\widetilde{w}_{t+1}-w\right\rangle
-\left\langle \ln w_{t},w_t-\widetilde{w}_{t+1}\right\rangle \\
&=\left\langle w, \ln \frac{\widetilde{w}_{t+1}}{\widetilde{w}_{t}}\right\rangle-\left\langle w_t, \ln \frac{w_t}{\widetilde{w}_{t}}\right\rangle-\left\langle \widetilde{w}_{t+1}, \ln \frac{\widetilde{w}_{t+1}}{w_t}\right\rangle, 
\end{aligned}
\end{equation*}
then we have 
\begin{equation*}
\begin{aligned}
\left\langle \theta \ell_t,w_t-w\right\rangle
&\leqslant\frac{\theta^2}{2}\left\lVert \ell_t-\widehat{\ell}_t\right\rVert_\infty^2+\left\langle w, \ln \frac{\widetilde{w}_{t+1}}{\widetilde{w}_{t}}\right\rangle-\left\langle w_t, \ln \frac{w_t}{\widetilde{w}_{t}}\right\rangle \notag \\
&\leqslant\frac{\theta^2}{2}\left\lVert \ell_t-\widehat{\ell}_t\right\rVert_\infty^2+\left\langle w, \ln \frac{\widetilde{w}_{t+1}}{\widetilde{w}_{t}}\right\rangle-\frac{1}{2}\left\lVert w_t-\widetilde{w}_{t}\right\rVert_1^2
\end{aligned}
\end{equation*}
according to \cref{-Entropy-strongly-convex}, and thus,
\begin{equation*}
\begin{aligned}
\sum_{t=1}^T\left\langle \ell_t,w_t-w\right\rangle
&\leqslant
\frac{1}{\theta}\sum_{t=1}^T\left\langle w, \ln \frac{\widetilde{w}_{t+1}}{\widetilde{w}_{t}}\right\rangle
+\frac{\theta}{2}\sum_{t=1}^T\left\lVert \ell_t-\widehat{\ell}_t\right\rVert_\infty^2-\frac{1}{2\theta}\sum_{t=1}^T\left\lVert w_t-\widetilde{w}_{t}\right\rVert_1^2 \\
&\leqslant\frac{1}{\theta}\left\langle w, \ln \frac{w}{\widetilde{w}_{1}}\right\rangle
+\frac{\theta}{2}\sum_{t=1}^T\left\lVert \ell_t-\widehat{\ell}_t\right\rVert_\infty^2-\frac{1}{2\theta}\sum_{t=1}^T\left\lVert w_t-\widetilde{w}_{t}\right\rVert_1^2
\end{aligned}
\end{equation*}
since
\begin{equation*}
\begin{aligned}
\sum_{t=1}^T\left\langle w, \ln \frac{\widetilde{w}_{t+1}}{\widetilde{w}_{t}}\right\rangle
=\left\langle w, \ln \frac{\widetilde{w}_{T+1}}{\widetilde{w}_{1}}\right\rangle
&=\left\langle w, \ln \frac{w}{\widetilde{w}_{1}}\right\rangle-\left\langle w, \ln \frac{w}{\widetilde{w}_{T+1}}\right\rangle\\
&\leqslant \left\langle w, \ln \frac{w}{\widetilde{w}_{1}}\right\rangle-\frac{1}{2}\left\lVert w-\widetilde{w}_{T+1}\right\rVert_1^2
\leqslant\left\langle w, \ln \frac{w}{\widetilde{w}_{1}}\right\rangle.
\end{aligned}
\end{equation*}
\end{proof}

\section{Proof of \cref{ONES-corollary}}

\begin{proof}
Note that 
\begin{equation*}
\begin{aligned}
\left\lVert \ell_t-\widehat{\ell}_t\right\rVert_\infty^2
=\left\lVert \left\langle\overline{x}_t^{\varphi_t}-\widetilde{\overline{x}}_t^{\widehat{\varphi}_t}, \boldsymbol{x}_t\right\rangle\right\rVert_\infty^2
\leqslant\rho^2\left\lVert \overline{x}_t^{\varphi_t}-\widetilde{\overline{x}}_t^{\widehat{\varphi}_t}\right\rVert^2
\leqslant\rho^2\left(\left\lVert \overline{x}_t^{\varphi_t}-\overline{x}_t^{\widehat{\varphi}_t}\right\rVert+\left\lVert \overline{x}_t^{\widehat{\varphi}_t}-\widetilde{\overline{x}}_t^{\widehat{\varphi}_t}\right\rVert\right)^2,
\end{aligned}
\end{equation*}
where 
\begin{equation*}
\begin{aligned}
\left\lVert \overline{x}_t^{\widehat{\varphi}_t}-\widetilde{\overline{x}}_t^{\widehat{\varphi}_t}\right\rVert
\leqslant L\left\lVert \overline{x}_t-\widetilde{\overline{x}}_t\right\rVert
=L\left\lVert \left\langle w_t-\widetilde{w}_t, \boldsymbol{x}_t\right\rangle\right\rVert
\leqslant\rho L\left\lVert w_t-\widetilde{w}_t\right\rVert_1
\end{aligned}
\end{equation*}
since $\partial\widehat{\varphi}_t$ is $L$-Lipschitz continuous. 
According to \cref{ONES-regret}, the static regret upper bound for ONES is 
\begin{equation*}
\begin{aligned}
&\frac{1}{\theta}\sum_{i}w\left(i\right)\ln \frac{w\left(i\right)}{\widetilde{w}_1\left(i\right)}
+\theta\rho^2\sum_{t=1}^T\sup_{x\in C}\left\lVert x^{\varphi_t}-x^{\widehat{\varphi}_{t}}\right\rVert^2+\left(\theta\rho^4 L^2-\frac{1}{2\theta}\right)\sum_{t=1}^T\left\lVert w_t-\widetilde{w}_{t}\right\rVert_1^2 \\
=\ &\frac{1}{\theta}\sum_{i}w\left(i\right)\ln \frac{w\left(i\right)}{\widetilde{w}_1\left(i\right)}
+\theta\rho^2\sum_{t=1}^T\left(\sup_{x\in C}\left\lVert x^{\varphi_t}-x^{\widehat{\varphi}_{t}}\right\rVert^2+1_{\theta>\frac{1}{\sqrt{2}\rho^2 L}}\rho^2 L^2\left\lVert w_t-\widetilde{w}_{t}\right\rVert_1^2\right). 
\end{aligned}
\end{equation*}
\end{proof}

\section{Proof of \cref{expert-group}}
\label{pf:expert-group}

\begin{proof}
Note that 
\begin{equation*}
\exists\,j\in\left\{0, 1, \cdots, \left\lfloor\log_2\sqrt{2T-1}\right\rfloor\right\}\eqqcolon E,
\quad\text{such that}\quad
\dot{\eta}\in\frac{\rho}{\varrho\sqrt{M_T}}\Big[2^j, 2^{j+1}\Big),
\end{equation*}
then expert $e_j$ reaches the following almost optimal regret upper bound, 
\begin{equation}
\label{ejs-regret}
\begin{aligned}
\frac{\rho\left(\rho+2 P_T\right)}{2\eta_j}+\frac{\eta_j\varrho^2}{2}M_T
<\frac{\rho\left(\rho+2 P_T\right)}{\dot{\eta}}+\frac{\dot{\eta} \varrho^2}{2}M_T=\frac{3}{2}\varrho\sqrt{\rho\left(\rho+2 P_T\right)M_T}.
\end{aligned}
\end{equation}
Substitute \cref{ONES-regret-bound-rewrite,ejs-regret} into \cref{regret-decomposition} yields 
\begin{equation*}
\begin{aligned}
\sum_{t=1}^{T}\varphi_t\left(\overline{x}_t\right)-\varphi_t\left(z_t\right)
<\frac{-\ln \widetilde{w}_1\left(j\right)}{\theta}+\frac{\theta\rho^{2}\varrho^{2}}{2}M_T+\frac{3}{2}\varrho\sqrt{\rho\left(\rho+2 P_T\right)M_T}.
\end{aligned}
\end{equation*}
To determine this upper bound, it suffices to choose some appropriate $\widetilde{w}_1$ and $\theta$. 
Let $\widetilde{w}_1\left(i\right)=\beta\left(i+2\right)^{-\alpha}$, 
where $\alpha\geqslant\zeta^{-1}\left(2\right)$, $\beta^{-1}=\sum_{i\in E}\left(i+2\right)^{-\alpha}$. $\zeta^{-1}\left(2\right)\approx 1.728647238998183$ is the root of equation $\zeta\left(\alpha\right)=2$ on $\mathbb{R}_+$, $\zeta$ represents the Riemann $\zeta$ function, i.e., 
\begin{equation*}
\begin{aligned}
\zeta\left(\alpha\right)=\sum_{n=1}^{\infty}\frac{1}{n^\alpha},\quad \alpha>0. 
\end{aligned}
\end{equation*}
Note that $\beta>1$ and $\eta_j\leqslant\dot{\eta}$, we have 
\begin{equation*}
\begin{aligned}
-\ln \widetilde{w}_1\left(j\right)<\alpha\ln\left(j+2\right) \quad\text{and}\quad j\leqslant\log_2\sqrt{1+\frac{2 P_T}{\rho}}.
\end{aligned}
\end{equation*}
Thus, 
\begin{equation}
\label{upper-bound}
\begin{aligned}
\sum_{t=1}^{T}\varphi_t\left(\overline{x}_t\right)-\varphi_t\left(z_t\right)
<\frac{\alpha}{\theta}\ln\left(2+\log_2\sqrt{1+\frac{2 P_T}{\rho}}\right)+\frac{\theta\rho^{2}\varrho^{2}}{2}M_T+\frac{3}{2}\varrho\sqrt{\rho\left(\rho+2 P_T\right)M_T}.
\end{aligned}
\end{equation}
Let $\theta\propto\frac{1}{\sqrt{M_T}}$, we have 
\begin{equation*}
\begin{aligned}
\sum_{t=1}^{T}\varphi_t\left(\overline{x}_t\right)-\varphi_t\left(z_t\right)
<O\left(\sqrt{\left(1+P_T\right)M_T}\right).
\end{aligned}
\qedhere
\end{equation*}
\end{proof}

\section{Proof of \cref{adaptive-trick}}
\label{pf:adaptive-trick}

\begin{proof}
Suppose the game has been played for $T$ rounds, and is in stage $m$. $M_T\in\left[4^m, 4^{m+1}\right)$. Denote by $T_s$ the total rounds number have been played in  stage $s$. $T=\sum_{s=1}^{m}T_s$. According to \cref{upper-bound}, the dynamic regret upper bound of stage $s$ is 
\begin{equation*}
\begin{aligned}
O\left(\alpha\ln\left(2+\log_2\sqrt{1+\frac{2 P_{T_s}}{\rho}}\right)2^s+2\rho^{2}\varrho^{2}2^s+3\varrho\sqrt{\rho\left(\rho+2 P_{T_s}\right)}2^s\right)\leqslant O\left(\sqrt{\left(1+P_T\right)}2^s\right), 
\end{aligned}
\end{equation*}
and then 
\begin{equation*}
\sum_{s=1}^{m}O\left(\sqrt{\left(1+P_T\right)}2^s\right)=O\left(\sqrt{\left(1+P_T\right)}2^m\right)=O\left(\sqrt{\left(1+P_T\right)M_T}\right). 
\qedhere
\end{equation*}
\end{proof}

\section{Proof of \cref{2-expert-group}}

\begin{proof}
Similar to the proof of \cref{expert-group} (\cref{pf:expert-group}), we choose 
\begin{equation*}
\begin{aligned}
\widetilde{w}_1\left( e_\lambda\right)&=\beta\left(\lambda+2\right)^{-\alpha},&&\lambda\in \mathcal{E}, \\
\widetilde{w}_1\left(\epsilon_\mu\right)&=\beta\left(\mu+2\right)^{-\alpha},&&\mu\in \mathscr{E},
\end{aligned}
\end{equation*}
where $\alpha\geqslant\zeta^{-1}\left(1.5\right)$, $\beta^{-1}=\sum_{\lambda\in \mathcal{E}}\left(\lambda+2\right)^{-\alpha}+\sum_{\mu\in \mathscr{E}}\left(\mu+2\right)^{-\alpha}$. $\zeta^{-1}\left(1.5\right)\approx 2.185285451787483$ is the root of equation $\zeta\left(\alpha\right)=1.5$ on $\mathbb{R}_+$, $\zeta$ represents the Riemann $\zeta$ function. 
Let $\theta\propto\frac{1}{\sqrt{V_T}}$, $\theta\leqslant\frac{1}{\sqrt{2}\rho^2 L}$, we have  
\begin{equation*}
\begin{aligned}
\sum_{t=1}^{T}\varphi_t\left(\overline{x}_t\right)-\varphi_t\left(z_t\right)
\leqslant\begin{cases}
O\left(\sqrt{\left(1+P_T\right)\left(V_T+D_T\right)}\right), &\\
O\left(\sqrt{\left(1+P_T\right)V_T}\right), &\displaystyle\dot{\eta}\leqslant\frac{1}{\sqrt{2}L},
\end{cases}
\end{aligned}
\end{equation*}
where $\dot{\eta}=\sqrt{\rho\left(\rho+2 P_T\right)/\left(2 V_T\right)}/\varrho$. 
Recombine the above piecewise bounds, we have
\begin{equation*}
\begin{aligned}
\sum_{t=1}^{T}\varphi_t\left(\overline{x}_t\right)-\varphi_t\left(z_t\right)
\leqslant O\left(\sqrt{\left(1+P_T\right)\left(V_T+1_{L^2\rho\left(\rho+2 P_T\right)\leqslant\varrho^2 V_T}D_T\right)}\right).
\end{aligned}
\qedhere
\end{equation*}
\end{proof}

\end{document}